\documentclass[twoside]{IEEEtran}
\ifCLASSINFOpdf
   \usepackage[pdftex]{graphicx}
   \graphicspath{{../pdf/}{../jpeg/}}
   \DeclareGraphicsExtensions{.pdf,.jpeg,.png,.jpg}
\else
\fi

\usepackage{cite}

\usepackage[latin9]{inputenc}
\PassOptionsToPackage{ruled,linesnumbered,vlined}{algorithm2e}
\usepackage{soul}
\usepackage{color}
\usepackage[english]{babel}
\usepackage{array}
\usepackage{mathrsfs}
\usepackage{multirow}
\usepackage{algorithm2e}
\usepackage{amsmath}
\usepackage{amsthm}
\usepackage{amssymb}
\usepackage{graphicx}
\usepackage[unicode=true]{hyperref}
\usepackage{subfig}
\usepackage{booktabs}

\providecommand{\tabularnewline}{\\}

\theoremstyle{plain}
\newtheorem*{thm*}{\protect\theoremname}
\theoremstyle{definition}
\newtheorem{defn}{\protect\definitionname}
\theoremstyle{remark}
\newtheorem*{rem*}{\protect\remarkname}
\theoremstyle{plain}
\newtheorem{thm}{\protect\theoremname}
\theoremstyle{plain}
\newtheorem{cor}{\protect\corollaryname}
\theoremstyle{plain}
\newtheorem{lem}{\protect\lemmaname}
\theoremstyle{plain}
\newtheorem{prop}{\protect\propositionname}

\providecommand{\corollaryname}{Corollary}
\providecommand{\definitionname}{Definition}
\providecommand{\lemmaname}{Lemma}
\providecommand{\remarkname}{Remark}
\providecommand{\theoremname}{Theorem}
\providecommand{\propositionname}{Proposition}

\DeclareMathOperator*{\argmin}{argmin}
\usepackage[dvipsnames]{xcolor}

\begin{document}
%

\title{A Neighbor-Searching Discrepancy-based Drift Detection Scheme for Learning Evolving Data}

%
%

\author{Feng~Gu,
       Jie~Lu\IEEEauthorrefmark{1},~\IEEEmembership{Fellow,~IEEE,}
        Zhen~Fang,
        Kun~Wang,
        and Guangquan~Zhang
        

\thanks{Feng Gu, Jie Lu, Zhen Fang, Kun Wang, and Guangquan Zhang are with Australia Artificial Intelligence Institute (AAII), Faculty of Engineering and Information Technology, University of Technology Sydney, P.O. Box 123, Broadway NSW, Australia.
(email: geofgu@gmail.com, jie.lu@uts.edu.au, zhen.fang@uts.edu.au, kun.wang@uts.edu.au, guangquzn.zhang@uts.edu.au)}
\thanks{\IEEEauthorrefmark{1} Corresponding author.}


}

\IEEEtitleabstractindextext{%
\begin{abstract}
Uncertain changes in data streams present challenges for machine learning models to dynamically adapt and uphold performance in real-time. Particularly, classification boundary change, also known as real concept drift, is the major cause of classification performance deterioration. However, accurately detecting real concept drift remains challenging because the theoretical foundations of existing drift detection methods - two-sample distribution tests and monitoring classification error rate, both suffer from inherent limitations such as the inability to distinguish virtual drift (changes not affecting the classification boundary, will introduce unnecessary model maintenance), limited statistical power, or high computational cost. Furthermore, no existing detection method can provide information on the trend of the drift, which could be invaluable for model maintenance. This work presents a novel real concept drift detection method based on Neighbor-Searching Discrepancy, a new statistic that measures the classification boundary difference between two samples. The proposed method is able to detect real concept drift with high accuracy while ignoring virtual drift. It can also indicate the direction of the classification boundary change by identifying the invasion or retreat of a certain class, which is also an indicator of separability change between classes. A comprehensive evaluation of 11 experiments is conducted, including empirical verification of the proposed theory using artificial datasets, and experimental comparisons with commonly used drift handling methods on real-world datasets. The results show that the proposed theory is robust against a range of distributions and dimensions, and the drift detection method outperforms state-of-the-art alternative methods.
\end{abstract}

}

\maketitle

\IEEEdisplaynontitleabstractindextext

\IEEEpeerreviewmaketitle

\section{Introduction}\label{sec:introduction}

In recent years, data stream learning under uncertainty has attracted considerable attention \cite{ Lu2019Learning}. In streaming data, the data distribution may change over time, that is concept drift, which includes four categories: sudden drift, incremental drift, gradual drift, and reoccurrence drift \cite{ Lu2019Learning}. The uncertain occurrence of concept drift challenges traditional machine learning methods to adapt promptly, resulting in decreased learning performance \cite{Liu2022Concept, liu2020concept}. That is, the input data is a continuous stream and arbitrary distribution change may occur at any time, resulting in an increased rate of prediction error \cite{hulten01minin.c5}. And this issue widely occurs in many real-world applications \cite{chiang01fault.i1,wei02turnin.c7}. Therefore, it is needed to identify and handle concept drift in time to help maintain the learning performance on evolving data \cite{Yu2022Continuous, zhou2023multi}.

A relatively large body of literature addresses the issue of learning with non-stationary data streams by developing learning algorithms that are responsive to concept drift. These algorithms can be categorized into three strategies: 1) detection and retraining algorithms \cite{bach08paired,alippi13just}, which retrain a new model using the latest data to replace the old model; 2) adaptive models \cite{gama03accur,yang12increm}, mostly decision tree-based methods, which partially update leaf nodes when their performance deteriorates; and 3) adaptive ensembles \cite{gomes17adapt,jiao2024incremental,Pratama2020online}, which dynamically add new base classifiers to adapt to new concepts. However, for all these methods to be effective, accurately detecting drift, either explicitly or implicitly, is a prerequisite. On one hand, an insensitive detector might cause a late or even missed model update, thus failing to adapt to drift; on the other hand, an oversensitive detector could introduce computational overhead or impair the model's generalization capability by frequently triggering unnecessary model updates.

Existing concept drift detection methods fall into two groups based on their theoretical foundations: monitoring learner outputs and two-sample distribution tests \cite{lu14cmp_c}. Methods in the first group, including \cite{yasum07quick_c41,li2009concept_c42}, treat the learning error of the base classifier output as a random variable from Bernoulli trials, so that binomial distribution can be used to describe the distribution of the error rate. A significant increase in the error rate, when exceeding a preset threshold, indicates concept drift. Such methods are computationally efficient and are able to detect real concept drift because they monitor prediction output directly. However, they fall short of being able to describe the drift, thus are only suitable for rebuilding a new model rather than updating an existing one.

Methods in the second group apply two-sample tests (either existing or newly developed) on the latest data and previous data to decide whether their distribution differs. For practical purposes, multivariate non-parametric tests are preferable because in real-world applications, input data is often multi-dimensional and knowledge about the data distribution is not available a priori. These methods may be statistical divergence-based \cite{dasu06kdqtree}\cite{yang2019novel}, distribution test-based \cite{alippi08just_i}\cite{ijcai2018yu}, kernel-based \cite{gretton12mmd}, entropy-based \cite{vorburger06entrop}, computational geometry-based \cite{liu11trian}, or instance-based \cite{biswas14nonpar,liu2018accumulating,lu14cmp_c,lu16edit}. These methods have several advantages over error-based detection methods - they are able to directly detect distribution difference without depending on classification models; they are be able to identify drift regions for partial model updates; and they can be applied directly to input data without true labels. A major limitation of these methods is that they cannot distinguish between real drift - classification boundary change and virtual drift - distribution changes not affecting classification result (e.g., distribution density changes or covariance shift).

To deal with these issues, we propose a novel instance-based real concept drift detection method, based on a new statistic - Neighbor-Searching Discrepancy (NSD), based on the theory of spatial statistics and nearest neighbor distribution. It is distribution-free (no prior distribution knowledge required), has an exact distribution (no re-sampling needed) and only detects real concept drift (ignoring virtual drift). The main contributions are shown below:

\begin{itemize}
\item A neighbor-searching discrepancy process has been developed and analyzed. A novel statistic is proposed with a detailed to measure the probability of an observed search volume difference for identifying neighboring points from two samples of the same distribution. And the process is distribution-free and unbiased, and does not require prior information about data distribution.

\item A real concept drift detection is designed to detect the classification gap change. It is computationally efficient since the drift confidence level can be calculated directly, instead of relying on re-sample estimation; it is also able to describe a classification boundary change as the invasion or retreat of a certain class and can measure its magnitude.

\item A detailed theoretical analysis is conducted. The concepts about neighbor searching and real concept drift detection have been clearly defined. And the detailed theoretical analysis about the neighbor-searching volume ratio is given to better support the proposed real concept drift detection method. 

\item The experiment on both synthetic and real-world data reflect the efficiency of the proposed real concept drift detection method. The proposed method can help maintain the learning performance, and can be directly applied to high dimensional data without losing effectiveness
\end{itemize}

To the best of our knowledge, it is the first attempt to use spatial statistics theory for concept drift detection. The novelty lies in the endeavor to extend spatial statistics theory, which concerns only geographical data, to higher dimensions. The proposed concepts and theorems around neighbor-searching discrepancy are mathematically general and could be applied to other problems, such as boundary detection. These results could also provide other researchers with a new perspective on the new statistically guaranteed properties of nearest neighbor-based methods.

This paper is organized as follows. Section \ref{sec:NeighborSearchDiscrepancy} proposes the theory related to Neighbor-Searching Discrepancy and discusses its properties. Section \ref{sec:DetectionMethod} develops the real concept drift detection method based on NSD. The Experiment results are given in Section \ref{sec:Experiment}. Section \ref{sec:Conclusion} concludes this study with a discussion of future work.

\section{A New Measure: Neighbor-Searching Discrepancy\label{sec:NeighborSearchDiscrepancy}}

This section proposes a novel statistic that measures the probability of an observed search volume difference for finding neighboring points from two samples of the same distribution.

\subsection{Preliminaries}

The basic assumption of this study is that data points, with an unknown distribution, can be regarded as generated from a binomial point process in multi-dimensional Euclidean space. Specifically, given a measurable subset $A\subseteq\mathbb{R}^{d}$, let $f$ be a density function on $A$ and $n\in\mathbb{N}$, a \emph{binomial point process (b.p.p.)} $\mathbf{X}=\{x_{1},x_{2},\ldots,x_{n}\}$ consists of $n$ independent and identically distributed (i.i.d.) points $x_{i}\in\mathbb{\mathbb{R}}^{d}$ with the common density $f$. A binomial point process has the following property:

Given any subset $S\subseteq A$, the number of points in $S$, denoted as the counting measure $N(S)$, has a binomial distribution,

\[
N(S)\sim\mathrm{binom}(n,p)
\]

where $p=\int_{S}f(x)\mathrm{d}x$.

Now denote the volume of $S$ as $V(S)$, then the \emph{intensity} of $S$, denoted as $\lambda(S)$, is the expected number of points per unit volume in $S$, thus $\lambda(S)=np/V(S)$. For a non-homogeneous b.p.p., $\lambda(S)$ may vary for different $S$. But for a homogeneous b.p.p., which means that the data points are uniformly distributed, $\lambda(S)$ is constant and can be denoted as $\lambda$.

The motivation for a new statistic is that, given two sets $X_{1,}X_{2}$ of sample points in $\mathbb{R}^{d}$, generated from an unknown distribution, let $S$ be a subset of a particular shape (e.g. a ball) containing $k_{1}$ point from $X_{1}$, if we observe $k_{2}$ points from $X_{2}$ that fall in $S$, we would like to know the probability of this happening under the hypothesis that $X_{1,}X_{2}$ are generated from the same distribution. For a general non-homogeneous b.p.p., a direct derivation of this statistic is difficult since $\lambda(S)$ can vary from place to place. In the following discussion we adopt an alternative path by first deriving the statistic for homogeneous b.p.p. and then showing that the result still holds for non-homogeneous b.p.p. and can be generalized to arbitrary shapes besides balls.

\subsection{Neighbor Searching}

A traditional searching for the k-nearest-neighbor of a given point $x$, in a set of sample points, is performed by ordering the sample points according to their distance to $x$, computed by a certain distance function. To discuss the property of the volume that is covered during this searching process, we need a more generalized definition for the k-nearest-neighbor searching.

\begin{defn}[Neighbor Searching]
 A \emph{ (neighbor) search} ${S}$ over $\mathbb{R}^d$ is a sequence of measurable sets $\{S_{i}\}$, called \emph{search steps}, with the following properties:

1) $S_{0}=\{x^{1},x^{2},\ldots,x^{m}\}$, where $x^i \in \mathbb{R}^d$, is the initial search step containing $m$ \emph{starting points};

2) $S_{i}\subset S_{j}$ if $i<j$, where $i,j\in\mathcal{I}$;

3) for any $i<j,i,j\in\mathcal{I}$, we have $V(S_{i})<V(S_{j})$;

4) $\sup_{i\in \mathcal{I}} V(S_i))=+ \infty$;

5) For any $v \in [0,+\infty)$, there exists $i \in \mathcal{I}$ such that $V(S_i)=v$;
\\
where $\mathcal{I}=[0,+\infty)$ is the index set.
\label{def:NeighborSearch}
\end{defn}








Based on this definition, we redefine the $k\mathrm{th}$ nearest neighbor as follows.

\begin{defn}[$k$th Nearest Neighbor]
  Given  a continuous distribution $P$, neighbor searching $S$ over $\mathbb{R}^d$ , and samples $X=\{x_1,...,x_n\}$ distributed by $P$, we say a set $S{(k)}\in S$ is the $k$th (nearest) neighbor if
  
  1) there exists an index $i\in \mathcal{I}$ such that $S{(k)}=S_i$;

  2) the number of samples in $S_i$ is $k$ ($\# X \bigcap S_i=k$);

  3) the number of samples in $S_j$ is smaller than $k$, where $j<i$ and $j\in \mathcal{I}$ ~($\# X \bigcap S_j<k$).
  
   The $k$th (nearest) neighbor-searching volume, denoted as $V(k) = V (S(k))$, is the volume of the $k$th nearest neighbor set $S(k)$.
\label{def:NearestNeighbor}
\end{defn}
  


  
\begin{rem*}
We can see that the traditional searching of neighbors for a point according to their distance is a special case of \emph{neighbor searching} where $S_{0}$ includes a single starting point and $S_{i}$ are concentric balls indexed by their radii, of which the 2D case is shown in Figure \ref{fig:NeighborSearch} (a). Also in Figure \ref{fig:NeighborSearch}, (b) and (c) demonstrate other possible shapes of the search; additionally, the shape of the search steps can also increase in an imbalanced manner as shown in (d) and (e); finally, (f) is the case of neighbor searching of multiple starting points. Notice that different neighbor searching does not necessarily yield the same result of nearest neighbors.
\end{rem*}

\begin{figure}
\subfloat[Sphere]
{\includegraphics[scale=0.4]{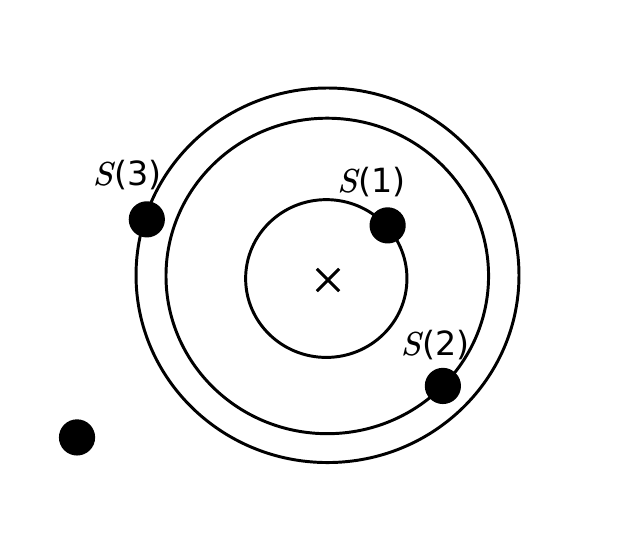}}
\hfill
\subfloat[Triangle]
{\includegraphics[scale=0.4]{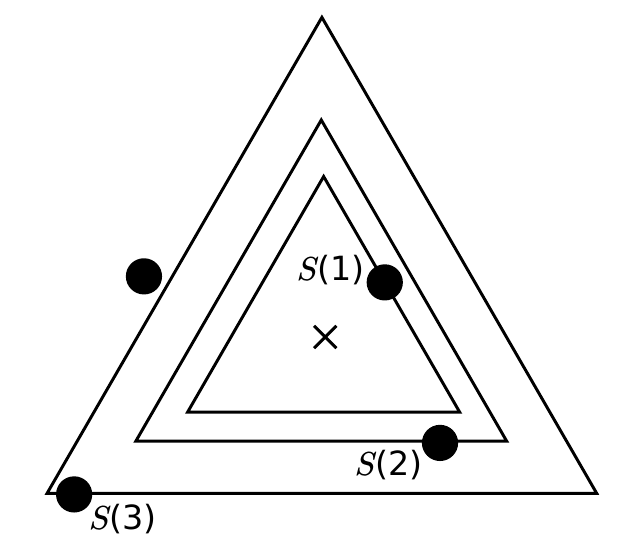}}

\subfloat[Irregular]
{\includegraphics[scale=0.4]{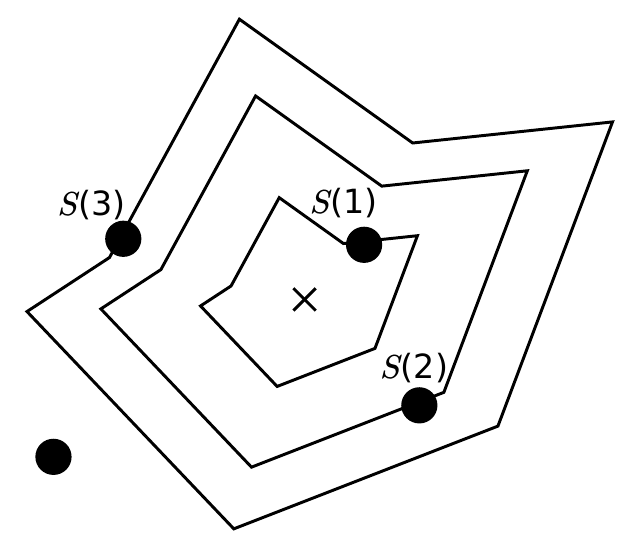}}
\hfill
\subfloat[Stepped]
{\includegraphics[scale=0.4]{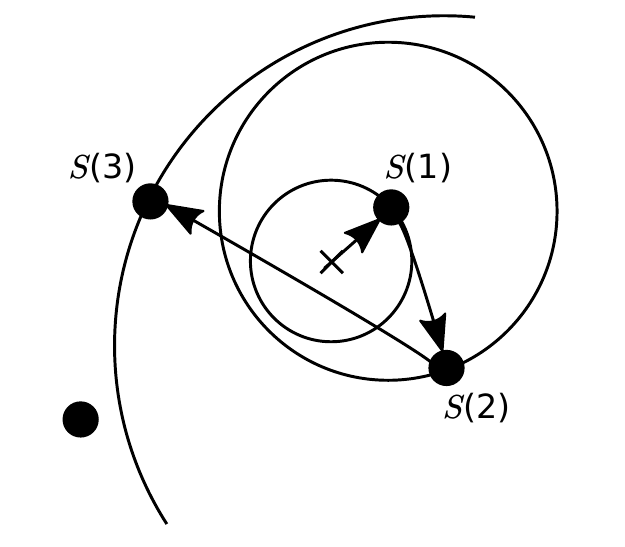}}

\subfloat[Imbalanced]
{\includegraphics[scale=0.4]{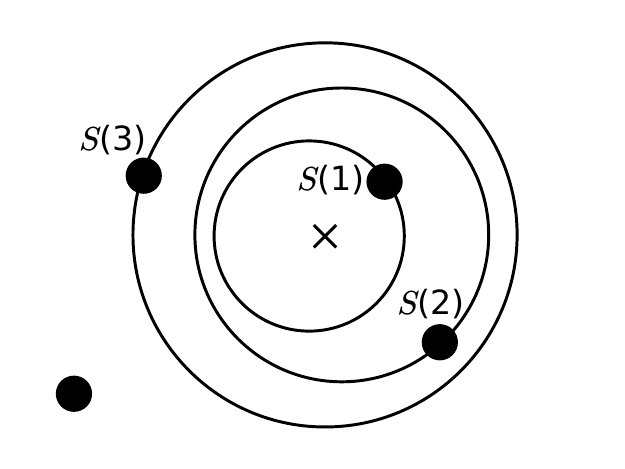}}
\hfill
\subfloat[Multiple starting points]
{\includegraphics[scale=0.4]{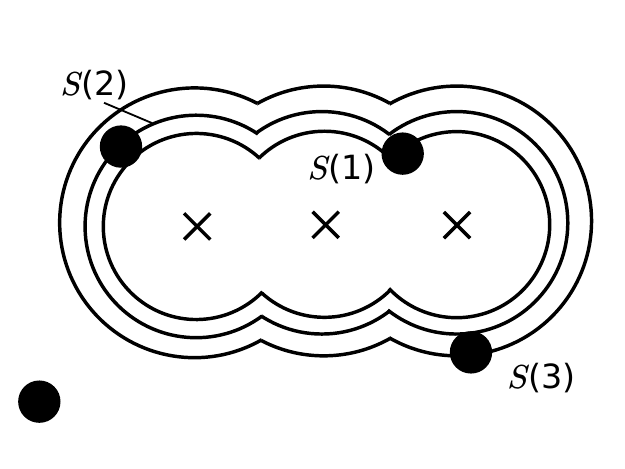}}
\caption{Examples of finding 3 nearest neighbors following neighbor searching of different manifold spaces.}
\label{fig:NeighborSearch}
\end{figure}

\begin{figure*}
\begin{centering}
\includegraphics[scale=0.22]{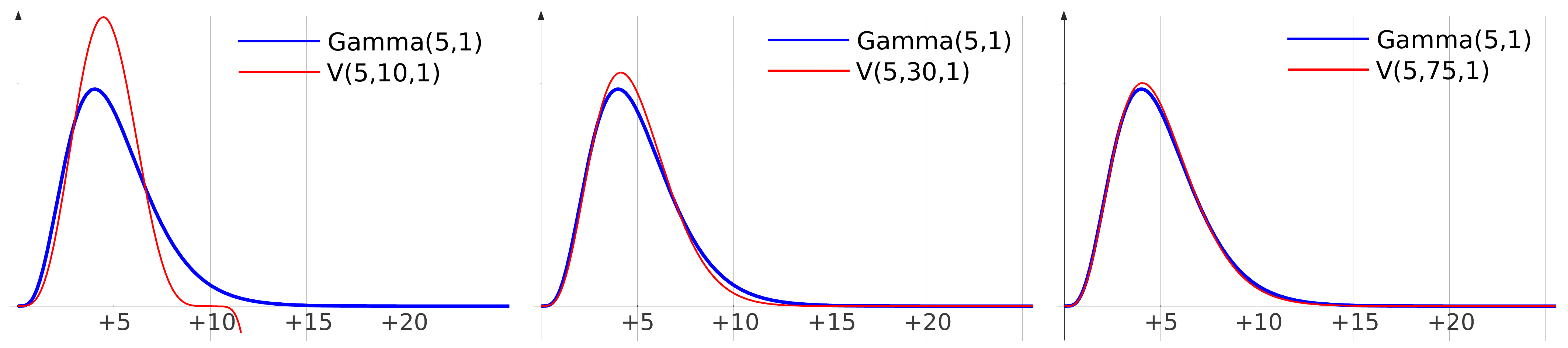}
\par\end{centering}
\caption{The PDF of $\mathcal{V}(k,n,\lambda)$ converges to that of $\mathrm{Gamma}(k,\lambda)$ as $n$ becomes greater than $k$.\label{fig:NSDamma}}
\end{figure*}

The distribution of neighbor-searching volume of concentric spheres in 2D, that is the area, has been investigated in spatial statistics \cite{ripley2005spatial}. The probability of finding the $k$th neighbor within a given circle of radius $t$ is given by
\begin{equation}
1-\sum_{j=0}^{k-1}\exp(-\lambda\pi t^{2})\frac{(\lambda\pi t^{2})^{j}}{j!}
\end{equation}
Our first step is to generalize a similar result to multiple dimensions. We know that if the sample points are generated from b.p.p., the distribution of the number of points falling in a certain region has a binomial distribution. Thus, for a homogeneous b.p.p in $\mathbb{R}^{d}$ with intensity $\lambda$, let $V(k)$ be the $k\mathrm{th}$ neighbor-searching volume of given neighbor searching, and consider $V(k)$ as a random variable, then with a given volume value $v \in [0,+\infty)$, we have
\begin{equation}
\mathbb{P}(V(k)>v)=\sum_{j=0}^{k-1}\binom{n}{j}\Big(\frac{\lambda v}{n}\Big)^{j}\Big(1-\frac{\lambda v}{n}\Big)^{n-j}\label{eq:VolumeDistribution}
\end{equation}
\begin{proof}
According to property 5) of  Definition 1, we can find a set $S\in \{S_i\}$, such that $V(S)=v$, then

\[
P(V(k)>v)=P(N(S)\leq k-1)=\sum_{j=0}^{k-1}P(N(S)=j)
\]

According to the property of b.p.p., for each integer $j\in[0,k-1]$,

\[
P(N(S)=j)=\binom{n}{j}\Big(p\Big)^{j}\Big(1-p\Big)^{n-j}
\]

where $p=\lambda v/n$.
\end{proof}
For random variable $V(k)$, we can directly obtain its cumulative density function (CDF) as

\begin{equation}
\mathbb{P}(V(k)\le v)=1-\sum_{j=0}^{k-1}\binom{n}{j}\Big(\frac{\lambda v}{n}\Big)^{j}\Big(1-\frac{\lambda v}{n}\Big)^{n-j}\label{eq:VolumeCDF},
\end{equation}
where $v \in [0,+\infty)$.
Then we derive the probability density function (PDF),

\begin{equation}
\mathbb{P}'(V(k)\le v)=\mathrm{binom}(k,n,\frac{\lambda v}{n})\cdot\frac{k}{v}\label{eq:VolumePDFbinom}
\end{equation}

For convenience, we define this distribution as follows.

\begin{defn}[Neighbor-searching Volume Distribution]
For a homogeneous b.p.p in $\mathbb{R}^{d}$ with intensity $\lambda$, let random variable $V(k)$ be the $k\mathrm{th}$ neighbor-searching volume in the given neighbor searching, then $V(k)$ follows \emph{neighbor-searching volume distribution}, denoted as $\mathcal{V}(k,n,\lambda)$.
\label{def:VolumeDistribution}
\end{defn}

\subsection{Neighbor-Searching Volume Ratio}

The previous section defined neighbor-searching volume distribution based on one sample of points. In this section, we expand this by investigating the relations between two random variables of neighbor-searching volume distribution based on two samples.

To simplify the calculation, we begin with the following lemma.

\begin{lem}
For a homogeneous b.p.p, given a random variable $V(k)\sim\mathcal{V}(k,n,\lambda)$, if
n is much greater than k, denoted as $n\gg k$, then $V(k)\sim\mathrm{Gamma}(k,\lambda)$.\label{lem:Gamma}
\end{lem}
\begin{proof}
See Appendix A.
\end{proof}

The amount by which $n$ is greater than $k$ determines the accuracy of the approximation. As shown in Figure \ref{fig:NSDamma}, $\mathcal{V}(k,n,\lambda)$ fits $\mathrm{Gamma}(k,\lambda)$ almost perfectly when $n\geq10\times k$. In a classification task, this prerequisite of $n\gg k$ is often easily met because the classification boundary is only a fraction of all the data points.

Now we propose a new statistic measure by considering the quotient of two random variables that follow the neighbor-searching volume distribution with the same $\lambda$ and $v$.

\begin{defn}[Neighbor-searching Volume Ratio]
Given two independent random variables $V(k_1)$$\sim\mathcal{V}(k_{1},n,\lambda)$,$V(k_2)\sim\mathcal{V}(k_{2},n,\lambda)$, define the random variable $\frac{V(k_1)}{V(k_1)+V(k_2)}$ as the \emph{neighbor-searching volume ratio}, denoted as $R_{(k_{1},k_{2})}$.\label{def:VolumeDistributionRatio}
\end{defn}

The direct derivation of the distribution of $R_{(k_{1},k_{2})}$ is difficult. However, with Lemma \ref{lem:Gamma}, we can obtain a simplified form of its distribution.

\begin{lem}
For a homogeneous b.p.p, if $n\gg k_{1}$ and $n\gg k_{2}$, then $R_{(k_{1},k_{2})}\sim\mathrm{Beta}(k_{1},k_{2})$.\label{lem:Beta}
\end{lem}
\begin{proof}
See Appendix A.
\end{proof}
This result still requires the homogeneity of data distribution, which means that the distribution of $R_{(k_{1},k_{2})}$ does not hold for arbitrary data distribution in real-world applications. However, as we will see, for the purpose of drift detection, we do not need the exact distribution of $R_{(k_{1},k_{2})}$, which offers the opportunity to eliminate dependency on the uniformity of $\lambda$.

\begin{thm}
Given any continuous distribution $P$ in $\mathbb{R}^d$, $\mathbb{P}(R_{(k_{1},k_{2})}<0.5)$ or equivalently $\mathbb{P}(V(k_1)< V(k_2))$ only depends on the number of samples $n_1$, $n_2$ and $k$, and doesn't depend on distribution $P$. \label{thm:Nonuniform}
\end{thm}
\begin{proof}
See Appendix B. 
\end{proof}

Definition \ref{def:VolumeDistribution}, \ref{def:VolumeDistributionRatio} and Theorem \ref{thm:Nonuniform} reveal an important fact: Given a set of sample points $X_{1}$ of an arbitrary distribution, we find the $k_{1}\mathrm{th}$ neighbor of certain neighbor searching and let the corresponding volume be $V(k_1)$. Now with another set of sample points $X_{2}$ of the same distribution, we find the $k_{2}\mathrm{th}$ neighbor and let the corresponding volume be $V(k_2)$. Then the probability of $V(k_1)>V(k_2)$ can be computed from the Beta distribution and the result does not depend on the value or uniformity of the intensity of the search steps of $V(k_1)$ and $V(k_2)$. We give a definition for this important statistic.

\begin{defn}[Neighbor-searching Discrepancy]
Given two sets of sample points $X_{1},X_{2}$ generated from a b.p.p in $\mathbb{R}^{d}$, for certain neighbor searching, let random variable $V(k_1)$ be the $k_{1}\mathrm{th}$ neighbor-searching volume in $X_{1}$ and $V(k_2)$ be the $k_{2}\mathrm{th}$ neighbor-searching volume in $X_{2}$, with random variable $R_{(k_{1},k_{2})}=V(k_1)/(V(k_1)+V(k_2))$. We define the \emph{neighbor-searching discrepancy} as
\[
\mathrm{NSD}(k_{1},k_{2})=\mathbb{P}(R_{(k_{1},k_{2})}<0.5)=\mathrm{CDF}_{\mathrm{Beta}(k_{1},k_{2})}(0.5).
\]
\label{def:NSD}
\end{defn}
\vspace{-0.3cm}
Note that the homogeneity assumption has been removed from this definition. Also, although the Beta distribution is continuous, we only consider $k_{1},k_{2}$ as positive integers since their original meaning is the number of sample points in a certain region.

As shown in Figure \ref{fig:Beta}, $\mathrm{NSD}$ is equal to the area under the curve of Beta PDF between $[0,0.5]$. When $k_{2}<k_{1}$, $\mathrm{NSD}(k_{1},k_{2})$ is small, meaning that the probability of $V(k_2)>V(k_2)$ is small, and vice versa. When $k_{1}=k_{2}$, $\mathrm{NSD}(k_{1},k_{2})$ is exactly 0.5. Also, note that $\mathrm{NSD}$ is not symmetric with respect to the difference between $k_{1}$ and $k_{2}$.

\begin{figure*}
\centering
\subfloat[$\mathrm{NSD}(10,5)$]
{\includegraphics[scale=0.7]{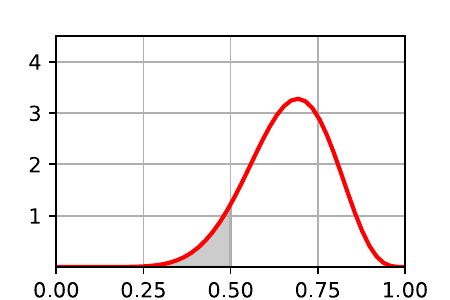}}
\hfill
\subfloat[$\mathrm{NSD}(10,10)$]
{\includegraphics[scale=0.7]{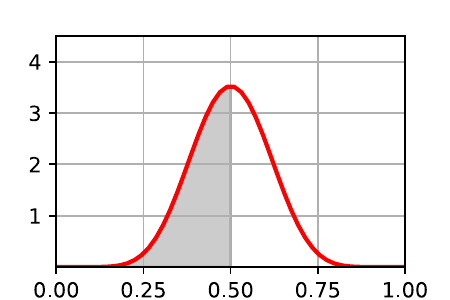}}
\hfill
\subfloat[$\mathrm{NSD}(10,15)$]
{\includegraphics[scale=0.7]{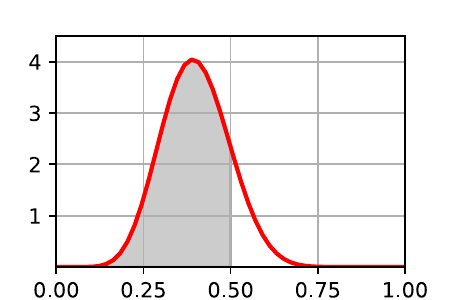}}
\caption{Neighbor-searching discrepancy equals area under the curve of Beta PDF between $[0,0.5]$.}
\label{fig:Beta}
\end{figure*}



\section{Real Concept Drift Detection\label{sec:DetectionMethod}}

In this section, we first introduce the concept of classification gap as a replacement for the traditional concept of classification boundary, which we avoid using because it is algorithm-specific and ambiguous; then, we develop a drift detection method based on the previous results of neighbor-searching discrepancy to detect the classification gap change. The new method does not rely on re-sampling techniques and is able to indicate the direction of the gap change.

\subsection{Classification Gap}

Real drift is defined in the literature as classification boundary change \cite{gama14survey_s}. However, this definition is of little help to standalone real drift detection due to the lack of an exact description of classification boundary. In practice, a classification boundary is a calculated line (in 2D or a higher dimensional plane) that is algorithm specific. For example, given a dataset of two classes, the classification boundaries generated by KNN and SVM could be significantly different. We therefore extend the traditional definition and consider real concept drift as a classification gap change. As shown in Figure \ref{fig:RealDrift}, traditional classification boundaries (dashed lines) are extended to the classification gap (solid lines) which is bounded against each class. The classification gap is not only able to reflect classification boundary change, but also has the ability to describe the moving direction of the change. Note that virtual drift, which has no impact on the classification result, should not affect the classification gap, as shown in Figure \ref{fig:RealDrift} (d). We define classification gap as following.

\begin{defn}[Classification Gap]
Given a sample space $A\subseteq \mathbb{R}^{d}$ of binary classes $\{+,-\}$, let $C^+$ and $C^-$ be the open subsets of $A$ where the distribution density is non-zero. That is, $C^+ = \{x: P(+|x)>0, x \in A\}, C^- = \{x: P(-|x)>0, x \in A\}$, and $\overline{C^+}, \overline{C^-}$ are their corresponding closures. Define the \emph{classification gap} as
\[
	\mathbb{G} = \mathbb{D}^+ \cap \mathbb{D}^-
\]
where
\[
	\mathbb{D}^+ = \{x: \argmin\limits_{x \in \overline{C^+}} dist(x, x^-), \forall x^-\in C^-\}
\]
and
\[
	\mathbb{D}^- = \{x: \argmin\limits_{x \in \overline{C^-}} dist(x, x^+), \forall x^+\in C^+\}
\]
for distance function $dist$.
\end{defn}

\begin{figure}
\centering
\subfloat[Original data]
{\includegraphics[scale=0.32]{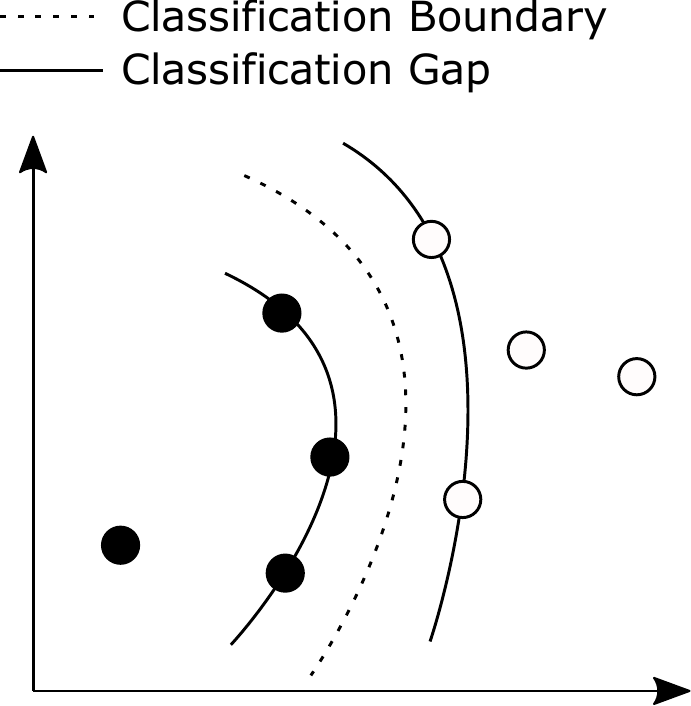}}
\hfill
\subfloat[Real drift (retreat)]
{\includegraphics[scale=0.32]{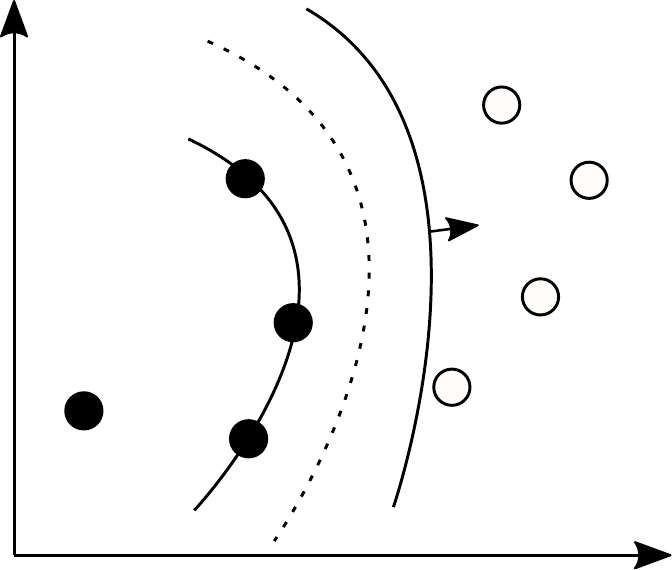}}
\hfill
\subfloat[Real drift (invasion)]
{\includegraphics[scale=0.32]{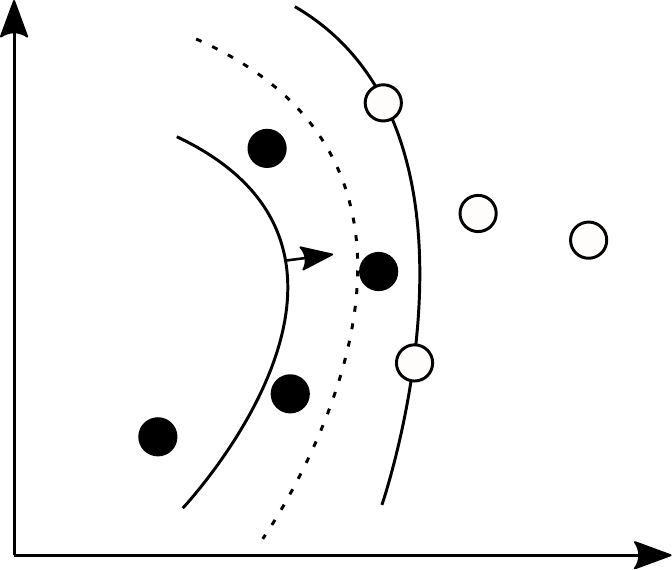}}
\hfill
\subfloat[Virtual drift]
{\includegraphics[scale=0.32]{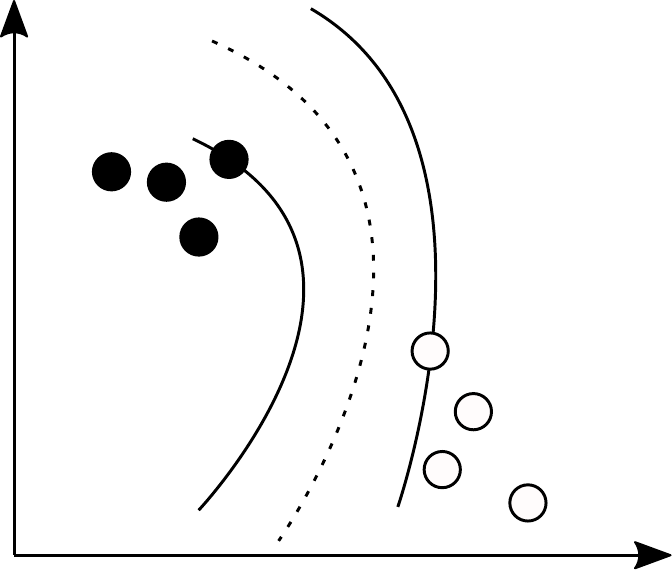}}
\caption{Different types of real concept drift and virtual drift reflecting on classification gap change.}
\label{fig:RealDrift}
\end{figure}

Next, with the help of neighbor-searching discrepancy, we are able to statistically detect and measure the classification gap change.

\subsection{Detect Classification Gap Changes}

Neighbor-searching discrepancy enables us to test if two sample sets have the same distribution in a sub-region. The prerequisite of this method is that the neighbor searching, that is, by definition, the set of search steps, must be predetermined to describe the region of interest. If we are to detect the classification gap change, the most important step is to define neighbor searching such that the shape of the search steps coincides with the shape of the classification gap, so that changes in the gap will reflect the probability of the instance counting difference within the search step, which is calculated by the neighbor-searching discrepancy.

\begin{figure}
\centering
\subfloat[Simple spherical search steps to the nearest neighbor of opposite class.]
{\includegraphics[scale=0.28]{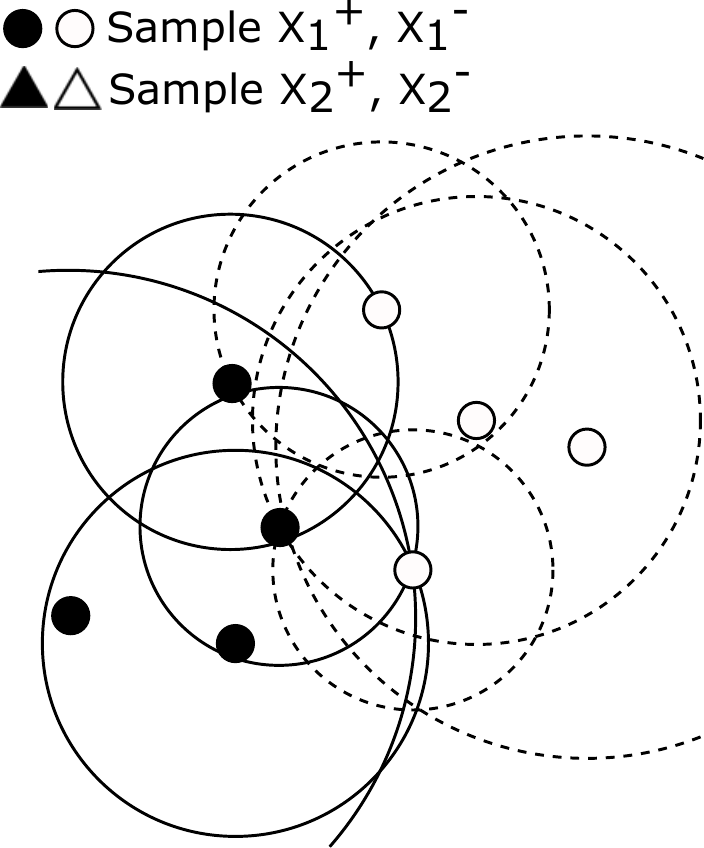}}
\hfill
\subfloat[Search step to approximate classification gap.]
{\includegraphics[scale=0.28]{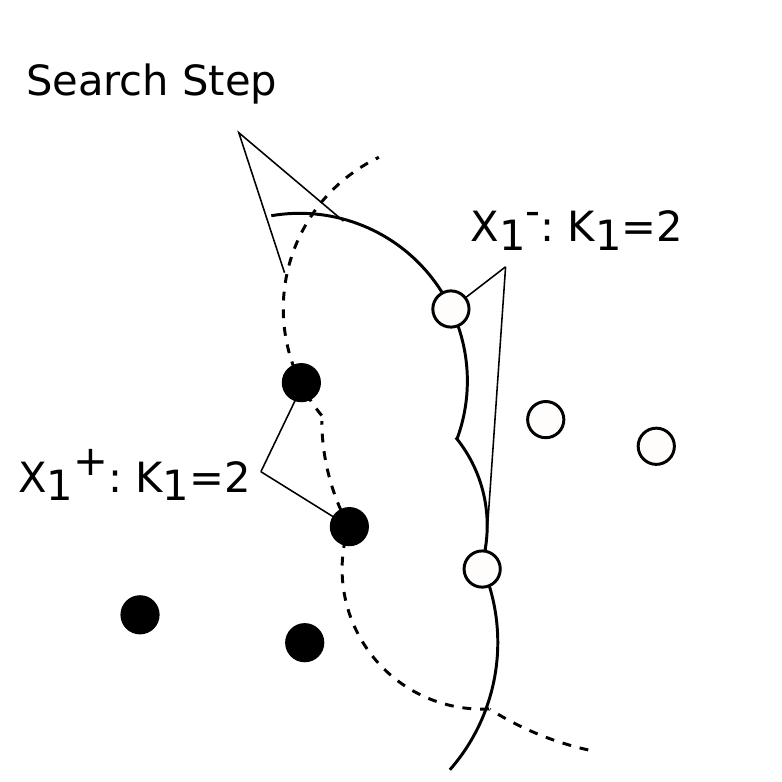}}
\hfill
\subfloat[Invasion and retreat of classification gap change.]
{\includegraphics[scale=0.28]{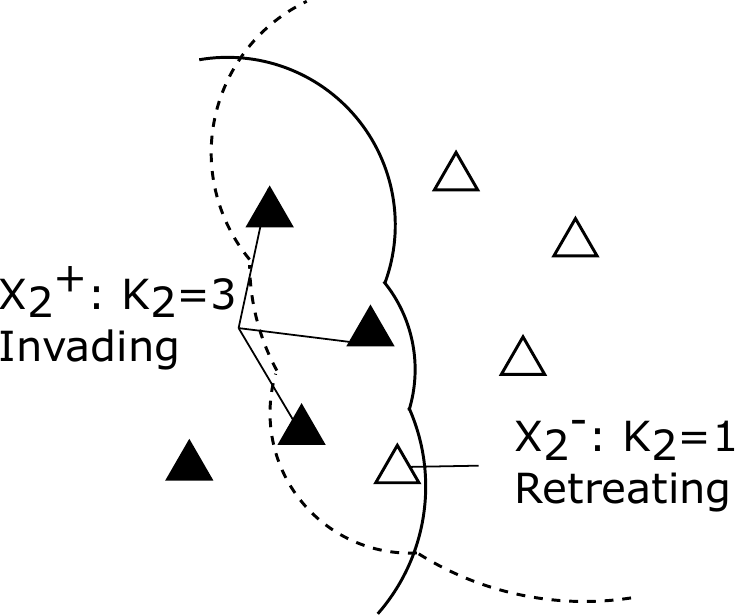}}
\caption{Neighbor searching is constructed by combining simple spherical searches on sample X1, and then invasion and retreat of classification gap change is tested on sample X2.}
\label{fig:RetreatInvasion}
\end{figure}

In practice, however, the classification gap of a dataset can be an arbitrary shape, thus simple neighbor searching, such as of spheres or hyperplanes, can hardly describe the shape accurately. However, according to Theorem \ref{thm:Nonuniform}, neighbor-searching discrepancy is a distribution-free statistic, which enables us to estimate the classification gap and describe its changes.

\begin{algorithm}[t]
\caption{Neighbor-Searching Discrepancy Drift Detection (NSD)}
\label{alg:1}
\SetKwInOut{Input}{input}
\SetKwInOut{Output}{output}
\Input{$O$, origin set containing starting points of neighbor searching;
\newline $R$, reference set for searching neighbors;
\newline $T$, test set for detecting gap change;
\newline $k$, number of neighbors for simple spherical neighbor-searching;
\newline $\theta$, confidence threshold for drift detection, default 0.05.
}
\Output{detection result for test set,
\newline 0 (no drift)
\newline 1 (drift invasion)
\newline 2 (drift retreat)
}
\BlankLine
initialize $P,Q:=$ empty sets\;
\ForAll{$x_i$ in $O$}{
	find $k$ nearest neighbors of $x_i$ from $R$, as $p_i$\;
	$P:=P \bigcup p_i$\;
	store the distance from $x_i$ to its $k$th neighbor, as $r_i$\;
}
$K_1:=$ size of $P$\;
\ForAll{$x_i$ in $O$}{
	find all points from $T$ whose distance to $x_i$ is smaller than $r_i$, as $q_i$\;
	$Q:=Q \bigcup q_i$\;
}
$K_2:=$ size of $Q$\;
\If{$K_1>K_2$ and $\mathrm{NSD}(K_1,K_2+1)<\theta$}{
	\KwRet{2 (drift retreat)}
}
\ElseIf{$K_1<K_2$ and $1-\mathrm{NSD}(K_1,K_2)<\theta$}{
	\KwRet{1 (drift invasion)}
}
\Else{
	\KwRet{0 (no drift)}
}
\end{algorithm}
The classification gap of a data set can be estimated by constructing neighbor searching in the following way. Given a binary classification sample set $X_1$, denote its points from different classes as $X_1^+$, $X_1^-$. Since the classification gap comprises two sides, the same procedure is conducted to estimate both sides. To demonstrate, we randomly choose one side, the one that is bounded by $X_1^-$. In this case, we denote $X_1^+$ as the \emph{origin set} $O$ and $X_1^-$ as the \emph{reference set} $R$. Now consider constructing a simple neighbor searching of spheres $\mathscr{S}^{(j)}=\{S_i^{(j)}\}$ for each point $j$ in $O$ with a common index set $\mathcal{I}=\{i\}$, such that $S_i^{(j)}$ has radius $C_{j}\cdot i$. The neighbor searching $\mathscr{S}^{\ast}$ is a new search with the same index $\mathcal{I}$, and the shape of its search steps is defined by constants $\{C_{j}\}$. Our goal, thus, becomes choosing a proper set of $C_{j}$ according to the shape of the classification gap bounded by $R$. This can be achieved by finding the $k$th neighbor of each $\mathscr{S}^{(j)}$ in $R$, for some user defined $k$, and measuring the radius of $S^{(j)}(k)$ as $r_{j}$. Without losing generality, let $i=1$ for this particular search step $S^{(j)}(k)$, then $C_{j}=r_{j}/i=r_{j}$. Figure \ref{fig:RetreatInvasion} (a) demonstrates this process for $k=1$. We can see in Figure \ref{fig:RetreatInvasion} (b), the search step $S^{\ast}$ is able to estimate the classification gap. Once gap estimation is done, we can test if there is a change in the gap by measuring another sample set $X_2$. In this case, $X_2^-$, denoted as the \emph{test set} $T$, has $k_2$ points falling within $S^{\ast}$. Compared with $k_1$ points from $R$ within $S^{\ast}$, the probability $P(S^{\ast}(k_2+1)>S^{\ast}(k_1))$ and $P(S^{\ast}(k_2)<S^{\ast}(k_1))$ can both be calculated by neighbor-searching discrepancy in Definition \ref{def:NSD}. Small values of these probabilities imply a classification gap change, as shown in Figure \ref{fig:RetreatInvasion} (c). The retreating and invading of the test set can be defined as follows.

\begin{defn}
\textbf{(Retreat and Invasion)} Let $S(k_1)$ be the search step that finds the $k_1$th neighbor in a reference set, let $k_2$ be the number of points from a test set $T$ that fall within $S(k_1)$, given a user-specified confidence threshold $\theta$, define \emph{retreat of $T$} as $\mathrm{NSD}(k_1,k_2+1)<\theta$ when $k_1\ge k_2$, and \emph{invasion of $T$} as $1-\mathrm{NSD}(k_1,k_2)<\theta$ when $k_1<k_2$.
\label{def:RetreatInvasion}
\end{defn}

\begin{rem*}
Based on the new indicators of retreat and invasion, it is possible to derive measurements from other perspectives of describing real drift; for example, widening/narrowing of a classification gap or separability of two classes.
\end{rem*}

Once we understand the role that neighbor searching plays behind the scenes which justifies the use of neighbor-searching discrepancy, the algorithm itself becomes straightforward, as shown in Algorithm \ref{alg:1}. Lines 1-6 apply kNN to construct $S^{\ast}(k_1)$ for gap estimation and store $\{r_{i}\}$ and $k_1$; lines 7-10 count the points from the test set in $S^{\ast}(k_1)$, which is done by comparing $r_{i}$ with the distance of each point from the test set to the corresponding point $x_i$ in the origin set; finally, lines 11-16 calculate the neighbor-searching discrepancy and return the test result. The procedure is conducted twice, first with $O=X_{1}^{+},R=X_{1}^{-},T=X_{2}^{-}$ for testing invasion/retreat of points of class $-$; and then with $O=X_{1}^{-},R=X_{1}^{+},T=X_{2}^{+}$ for testing invasion/retreat of points of class $+$. The complexity of the algorithm is $O(n^2k)$ for sample size $n$, given that the kNN operation is $O(nk)$, while updating one point has complexity $O(n)$.

\section{Experiment evaluation\label{sec:Experiment}}

Our evaluation of the proposed concept drift detection method consists of three sections, using 11 experiments:

\begin{enumerate}
\item \textbf{(Experiments 1-8)} The proposed statistic--Neighbor-Searching Discrepancy is the foundation of our drift detection method and could be potentially used in other instance-based problems. We thus design comprehensive experiments to demonstrate and verify its correctness and robustness against various data distributions and dimensions, using the Monte Carlo method.
\item \textbf{(Experiments 9, 10)} We apply the proposed drift detection method to synthetic datasets and compare our method to Lu et al. \cite{lu14cmp_c} for both accuracy and efficiency. The synthetic datasets are designed to simulate concept drift of different types and different degrees, which enables us to control the changes more easily and to see how our detection method performs against various changes.
\item \textbf{(Experiments 11)} We apply the proposed drift detection method to classification tasks on five real-world datasets and compare it with eight popular drift handling methods, including error rate-based methods, distribution test-based methods, as well as ensemble methods.
\end{enumerate}

The source code of the experiments can be downloaded from \href{https://github.com/geogu/cd.git}{https://github.com/geogu/cd.git}

\subsection{Evaluating the NSD of a single starting point}

In Section \ref{sec:NeighborSearchDiscrepancy}, we proposed an exact statistic, NSD, to measure the probability of observing an empirical classification boundary shift under the assumption of stationary data distribution. NSD not only serves as the foundation of the proposed detection method, the statistic itself is a general tool with the potential for developing new test methods. Thus, we conducted a series of experiments to demonstrate its effectiveness under various generated data distributions, from simple block-shaped distribution to multivariate Normal, Poisson and Gamma distribution, with dimensionality ranging from 2D to 1000D. We compared the calculated results to the Monte Carlo simulated result to show that they are asymptotically equal. For all experiments, the data point size and Monte Carlo sample size were both set to 1000 unless indicated otherwise.

\begin{figure*}
\centering
\subfloat[Starting point location: center, border, outer, remote.]
{\includegraphics[scale=0.28]{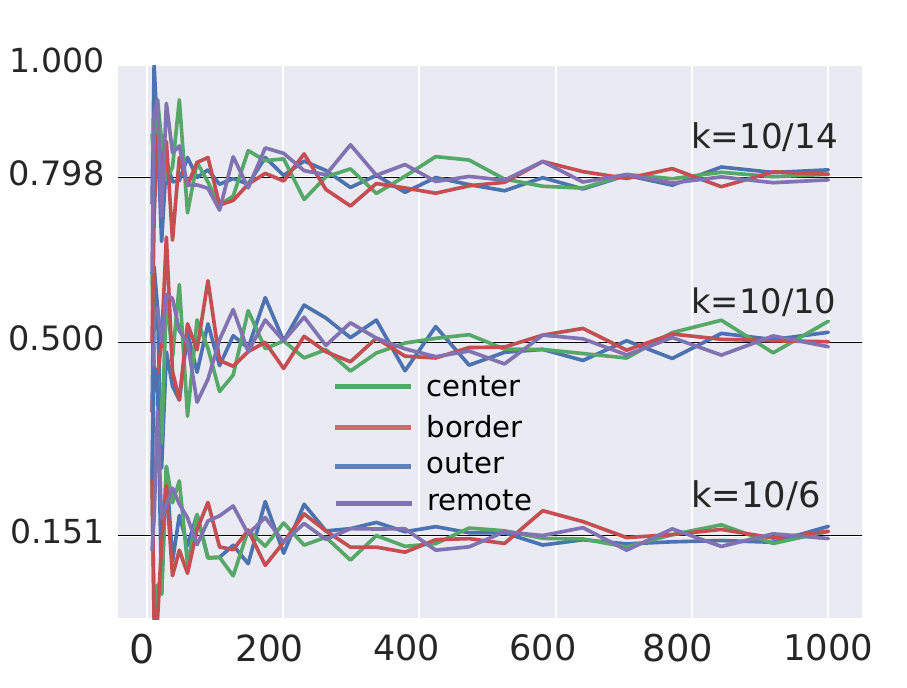}}
\subfloat[Data Dimensions: 10, 100, 1000.]
{\includegraphics[scale=0.28]{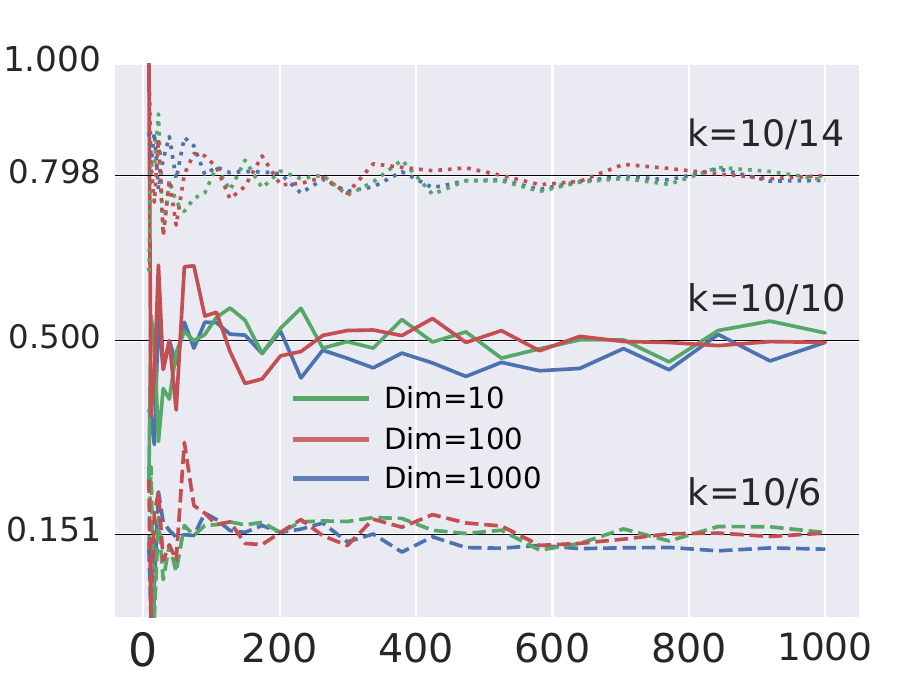}}
\subfloat[Normal distribution data.]
{\includegraphics[scale=0.28]{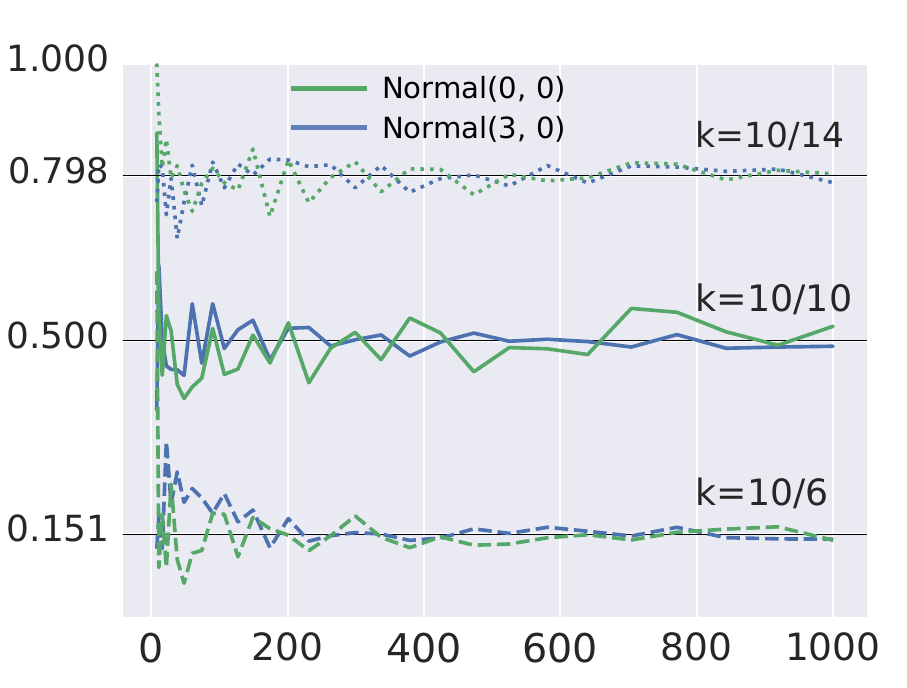}}
\subfloat[Gamma distribution data.]
{\includegraphics[scale=0.28]{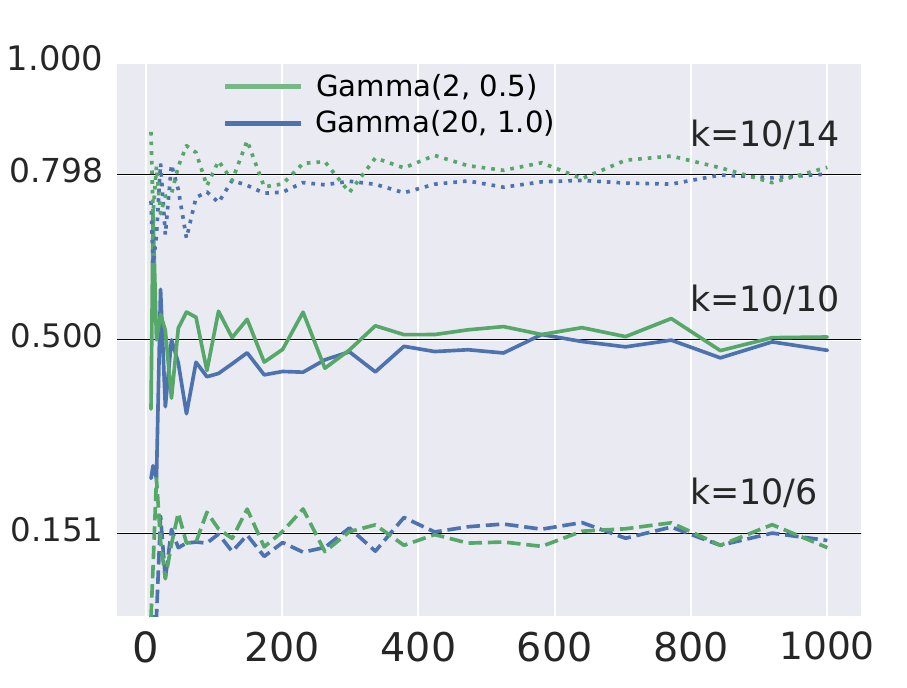}}

\subfloat[Poisson distribution data.]
{\includegraphics[scale=0.28]{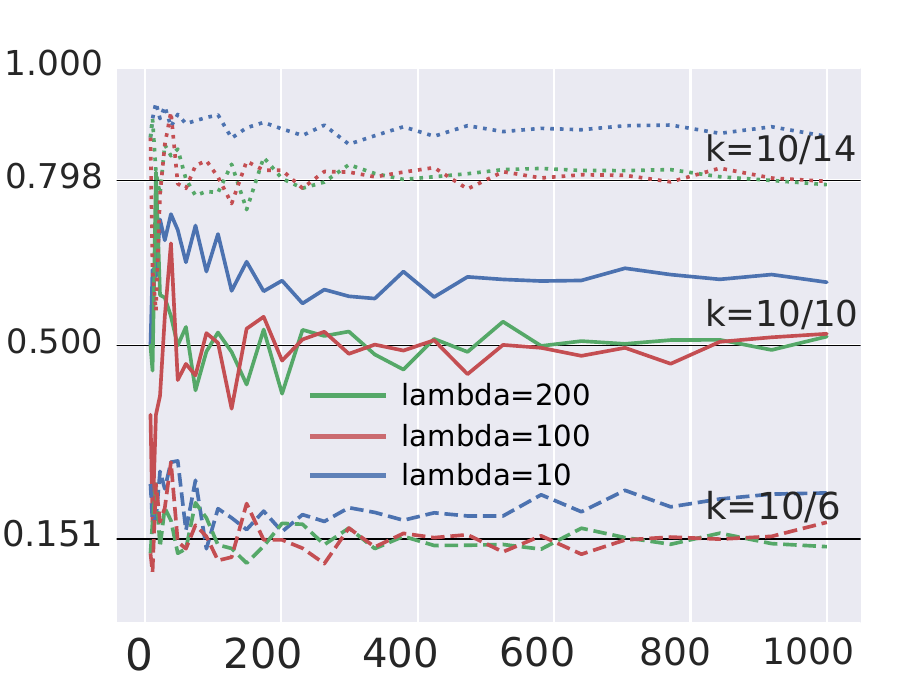}}
\subfloat[Disjoint search steps.]
{\includegraphics[scale=0.28]{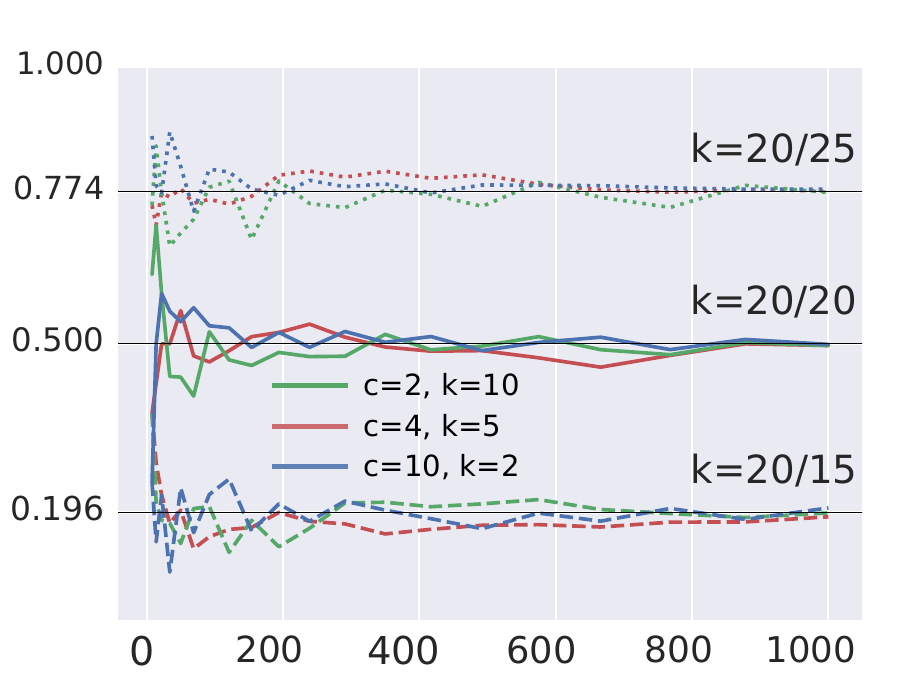}}
\subfloat[Overlapping search steps.]
{\includegraphics[scale=0.28]{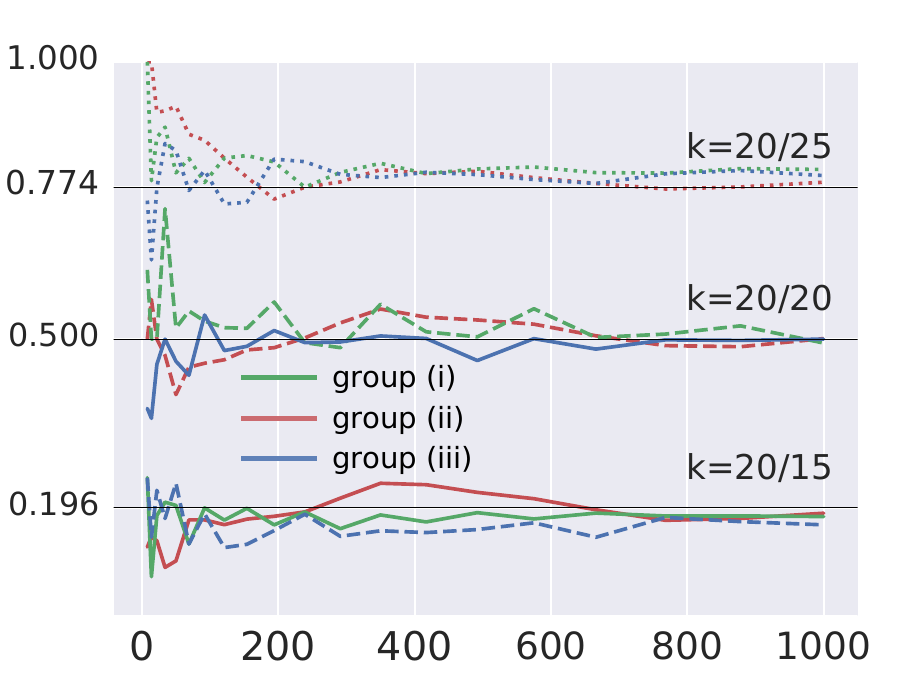}}
\subfloat[Stepped search steps.]
{\includegraphics[scale=0.25]{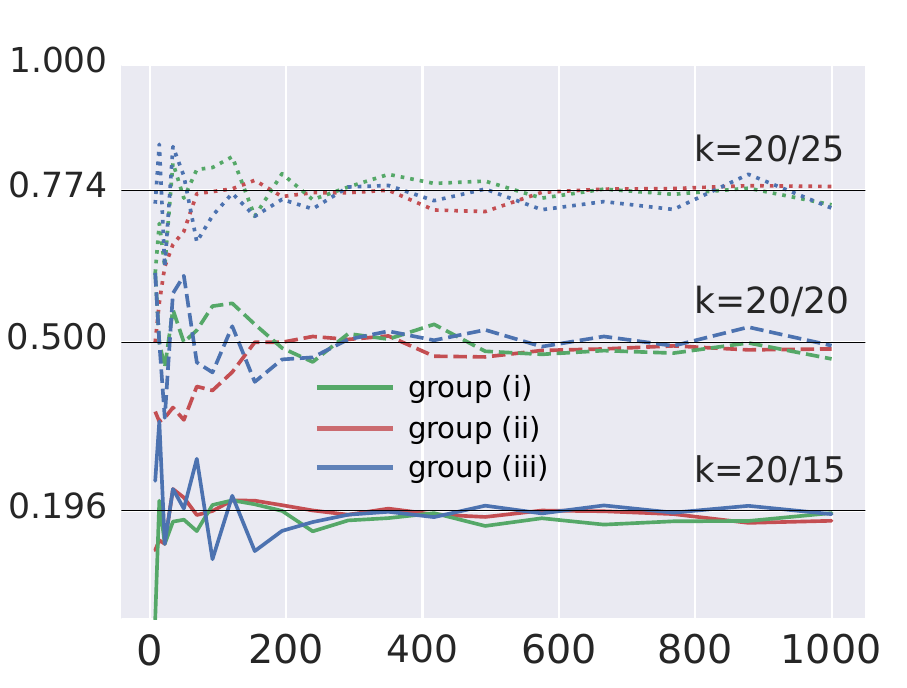}}
\caption{Comparison of the calculated NSD (black solid horizontal lines) with Monte Carlo simulation result (colored lines). The simulation result converges to NSD as sample size increases.
\label{fig:NSD}}
\end{figure*}

\subsubsection*{Experiment 1. NSD convergence rate - starting point location}

In this experiment, we begin with the simplest case of computing NSD and consider neighbor searching only of a single starting point. Data points are sampled from a 2D uniform\_(0,1) square. The starting point is positioned in a different vicinity to the sample points, namely, Center (0.5, 0.5); Border (1.0, 0.3); Outer (1.1, 0.3) and Remote (200.0, 0.3). For each starting point, we apply kNN to find the search step $S$ with the number of neighbors $k_{1}=10$ in the first batch of samples, then with the second batch of samples, we count the number of points that fall in $S$, denoted as $k_{2}$. This procedure is repeated 1000 times for each position of the starting point. We randomly choose $14,10,6$ and count the percentage of occurrence when $k_{2}$ is smaller than the selected value. Lastly, we plot the value of this percentage and compare it with the values of $\mathrm{NSD}(10,14)=0.798$, $\mathrm{NSD}(10,10)=0.5$, $\mathrm{NSD}(10,6)=0.151$, which are calculated directly from the formula of Definition \ref{def:NSD}. The result is shown in Figure \ref{fig:NSD}(a). As the number of repetitions increases, the percentage of $k_{2}<14,10,6$ asymptotically approaches NSD, which indicates that \emph{NSD is independent of the position of the starting point}. Also note that in the case of Remote (200.0, 0.3), the search step volume is approximately $124410.21$, which is much larger than the data points' distribution region area $[0, 1]$. This shows that \emph{NSD is independent of the neighbor-searching volume size}.

\subsubsection*{Experiment 2. NSD convergence rate - high dimension}

In this experiment, we apply NSD to data of different dimensions to verify its robustness to dimensionality. Data points are sampled from a multi-dimensional uniform(0,1) cube. The starting point is set to origin $(0,\ldots)$ for simplicity. The sample size is reduced to $500$ for ease of simulating a high dimensional situation where the dimensionality is much higher than sample size. We perform three groups of tests with dimensions $10,100,1000$. Similar to the previous experiment, we compute the search step $S$ for $k_{1}=10$ and count the percentage of occurrence when $k_{2}$ is smaller than $14,10,6$. The result is shown in Figure \ref{fig:NSD}(b). As the number of repetitions increases, the percentage of $k_{2}<14,10,6$ asymptotically approaches the calculated values of NSD, which indicates that \emph{NSD is independent of dimensionality,} as long as the prerequisite of $n\gg k$ is met. These results also show that the proposed statistic NSD is unique in the sense that it is based on Euclidean-distance yet is not affected by the curse of dimensionality.

\subsubsection*{Experiment 3.  NSD convergence rate - Normal distribution}

The following experiments demonstrate that the NSD is distribution-free. We start with a Normal distribution with different mean values. Data points are generated according to the 2D normal distribution with mean $\mu=0$ and $3.0$, fixed standard deviation $\sigma=1.0$. The starting point is set to $(0,0)$. For each mean value, similar to the previous experiment, we compute the search step $S$ for $k_{1}=10$ and count the percentage of occurrence when $k_{2}$ is smaller than $14,10,6$. The result is shown in Figure \ref{fig:NSD}(c). As the number of repetitions increases, the percentage of $k_{2}<14,10,6$ asymptotically approaches the calculated values of NSD, which indicates that \emph{NSD is not affected by the data distribution}.

\subsubsection*{Experiment 4.  NSD convergence rate - Gamma distribution}

This experiment is similar to Experiment 3, with data of Gamma distribution. The parameter of Gamma distribution is set to $\alpha=2,\beta=0.5$ and $\alpha=20,\beta=1.0$. The result is shown in Figure \ref{fig:NSD}(d). As the number of repetitions increases, the simulated value asymptotically approaches the calculated values of NSD, which indicates that \emph{NSD is not affected by the data distribution}.

\subsubsection*{Experiment 5. NSD convergence rate -  Poisson distribution}

In this experiment, we apply NSD to data with a Poisson distribution. Note that a Poisson distribution is a discrete distribution which is different from the previously mentioned continuous Normal and Gamma distribution. Discrete distribution impacts the computation of NSD, as shown in this experiment, since the search step volume can only take on discrete values in this situation. The data points are generated from multivariate Poisson distribution using the Trivariate Reduction method \cite{mardia1970families} with $(X,Y)\sim\mathrm{Poisson}(0.5\lambda,0.5\lambda,0.5\lambda)$. The parameter $\lambda$ is set to $10,100,200$. The test procedure is identical to the previous experiments. The result is shown in Figure \ref{fig:NSD}(e). As can be seen, the simulated value for $\lambda=100$ and $200$ (red and green lines) matches the calculated value of NSD. However, for $\lambda=10$ (blue line), the simulated value deviates from the calculated value of NSD by a large margin. The reason for this is that for a Poisson distribution, when $\lambda$ is small, data points can only have a few discrete locations. This violates the $n\gg k$ condition for neighbor-searching volume distribution, which leads to a deviated value of NSD. From this experiment, we see that in order to apply NSD to discrete data, extra attention is needed to make certain that the data has adequate dispersion.

\subsection{Evaluating NSD of multiple starting points}

The aforementioned experiments above demonstrate that NSD is distribution-free and dimension-independent, but they only consider the computation of NSD from a single starting point. Now, we focus on applying NSD to multiple starting points.

\subsubsection*{Experiment 6. Multiple starting points with non-overlapping search steps}

We first consider the simple case in which the multiple starting points are far away from each other. The search steps of these points should be non-overlapping if we choose a small enough $k$. The data points are sampled from a 2D uniform(0,1) square with sample size $n=1000$. We perform three groups of tests with different numbers of multiple starting points, denoted as $c$. The settings for each group are $(c=2,k=10)$, $(c=4,k=5)$ and $(c=10,k=2)$. The value of $k$ is chosen for convenience so that the number of neighbors in the search step is always the same, i.e., $k_{1}=c\times k=20$. For each group, the multiple starting points are positioned evenly across the $[(0,0),(1,1)]$ span to make sure their search step does not overlap. With the first batch of samples, we compute the search step $S$ with the number of neighbors $k_1$, then with the second batch of samples, we count the number of points $k_{2}$ that fall in $S$. After repetition, we count the percentage of occurrence when $k_{2}$ is smaller than the randomly chosen values $15,20,25$. The result is shown in Figure \ref{fig:NSD}(f). As the number of repetitions increases, the percentage of $k_{2}<15,20,25$ asymptotically approaches the calculated values $\mathrm{NSD}(20,15)=0.196$, $\mathrm{NSD}(20,20)=0.5$, $\mathrm{NSD}(20,25)=0.774$, which indicates that \emph{NSD is applicable for non-overlapping neighbor searching}.

\subsubsection*{Experiment 7. Multiple starting points with overlapping search steps}

This experiment verifies NSD on multiple starting points with overlapping search steps, meaning that computations on each point are not independent. The data points are sampled from a 2D uniform(0,1) square with sample size $n=1000$. The test procedure is similar to the previous experiment, the only difference being that the locations of the multiple starting points need to be close enough for their search step to overlap. We use three groups of multiple starting points with size 2, 4 and 10. These points are set as follows: i) (0.5,0.5), (0.536,0.5) for $k=15$; ii) (0.5,0.5), (0.53,0.5), (0.56,0.5), (0.586,0.5) for $k=10$; iii) (0,0.5), (0,0.52), (0,0.54), (0,0.56), (0,0.585), (1,0.5), (1,0.52), (1,0.54), (1,0.56), (1,0.585) for $k=5$. Because the search step $S$ may vary for different point samples when the search steps overlap, these points are carefully chosen, for clarity purposes, so that their $N(S)$ is approximately 20. The result is shown in Figure \ref{fig:NSD}(g). As the number of repetitions increases, the percentage of $k_{2}<15,20,25$ asymptotically approaches the calculated values of NSD, which indicates that \emph{NSD is applicable for overlapping neighbor searching}.

\subsubsection*{Experiment 8. Multiple starting points with stepped search}

In this experiment, we apply NSD in an interesting scenario in which the search steps of multiple starting points is calculated in a stepped manner, as demonstrated in Figure \ref{fig:NeighborSearch}(f). The data points are sampled from a 2D uniform(0,1) square as above. Three groups of multiple starting points are randomly chosen to represent a different degree of proximity to the sample points. Each group has four points and their values are: i) (0.5,0.5), (0.52,0.5), (0.54, 0.5), (0.56, 0.5) for the inner points; ii) (0.5, 0), (0.52, 0), (0.54, 0), (0.56, 0) for the border points; iii) (0.5, 10.5), (0.52, 10.5), (0.54, 10.5), (0.56, 10.5) for the outer points. The value of $k$ is set to 5. The search steps for each group of four points is computed sequentially so that the later points will always find five new neighbors that are not neighbors of the other three points. In this way, the value of $N(S)$ will always be 20. The result is shown in Figure \ref{fig:NSD}(h). As the number of repetitions increases, the percentage of $k_{2}<15,20,25$ asymptotically approaches the calculated values of NSD, which shows that NSD is applicable in this interesting scenario. This experiment also suggests that \emph{novel neighbor-searching strategies can be used as long as the neighbor searching conditions are met}.

\subsection{Evaluating the drift detection method}

The aforementioned experiments demonstrate that the NSD is a distribution-free and dimension-robust statistic and can be applied to single or multiple starting points, which forms the foundation of the proposed drift detection method. The following experiments demonstrate the effectiveness and efficiency of the proposed drift detection method. First, we compare our method with a state-of-the-art drift detection method in terms of both detection accuracy and computational efficiency, using synthetic datasets that are specially designed to present classification boundary change, i.e., real concept drift. Synthetic datasets are particularly useful for customizing a variety of concept drift cases, which helps us to determine the strengths and weaknesses of the proposed method. We then apply our drift detection method in classification tasks on real-world datasets and compare the final classification performance with several of the most popular concept drift tolerant classification methods.

\subsubsection*{Experiment 9. Classification boundary change detection}

In this experiment, we apply the proposed drift detection method to a group of synthetic datasets. We compare our method with another drift detection method, competence model-based concept drift detection (CM) proposed by Lu et al. \cite{lu14cmp_c}. We choose their method for comparison because first, CM is very sensitive to distribution change and is reported to have better accuracy than many other two-sample tests and distribution change detection methods, and second, unlike other error rate-based detection methods, our methods do not depend on a classifier and therefore have a broader application scope.

The datasets are designed to include different types of classification boundary drift. A total of nine datasets are used in this experiment, each of size 200,000, containing 99 drift locations. The datasets fall into 3 groups, as shown in Table \ref{tab:ExpSynthetic}. Datasets in the first group (No.1\textasciitilde 5) are generated in hyper-cube uniform(0,1) with a classification boundary specified by formula $Y=1$ when $\sum x_{i}>d\sqrt{\Delta^{2}/d}$ and $0$ otherwise, where $d$ is the number of dimensions and $\Delta$ is the offset value of the boundary shift. A larger value of $\Delta$ means a higher degree of drift. Datasets in the second group (No.6,7) are also generated in hyper-cube uniform(0,1). This time their classification boundary is rotated using the formula $Y=1$ when $x_{1}>x_{0}\cdot\mathrm{tan}(\pi\Delta/180)$ and $0$ otherwise, where $\Delta$ represents the degree of the rotation. Datasets in the third group (No.8,9) are generated according to 2D uncorrelated multivariate Normal distribution with standard deviation 1. The class label $Y=0$ if the points are sampled with mean $\mu=0$, and $Y=1$ if sampled with mean $\mu=\Delta$, where $\Delta$ represents the offset of the mean change, which results in classification boundary shift.

The tests are performed in a batch-by-batch manner, which is also known as the \emph{sliding window strategy}. The batch size (or window size) is set to 1000. The false alarm threshold, aka. false positive rate (FPR), is set to 5\% for both methods for a fair comparison. For our method, two different neighbor numbers $k=1,5$ are used to demonstrate the impact of choosing $k$ on different datasets. The permutation size for the CM method is set to $500$ as recommended in their original work \cite{lu14cmp_c}.

The results are shown in Table \ref{tab:ExpSynthetic}. As can be seen, the proposed method achieved higher detection accuracy on all test datasets. For the linear shift drift (No.1\textasciitilde 5), in particular, our method outperforms CM by a significant margin. This indicates that our method is very sensitive to small boundary shifts, whereas CM almost fails in such cases. We can also see that for the linear rotation tests (No.6,7), the advantageous margin of our method decreased. This is because when the classification boundary is rotating, one part of the sample points will be retreating while the other part of the sample point will be invading, which results to an extent in a canceling effect when counting points in the search step. It also shows that a smaller value of $k$ will result in better performance in this case, since the canceling effect will be more quickly overrun as the degree of rotation continues to increase. Lastly, by comparing the overall result of the proposed method with different $k$ values, we note that when the datasets have a clear classification boundary (No.1\textasciitilde 7), i.e., points of different classes are not overlapped, smaller $k$ leads to better accuracy, and when the classification boundary is not clear, or overlapped (No.8,9), a larger $k$ leads to better accuracy. This phenomenon is as expected because when the classification boundary is overlapped, a larger $k$ will cause the search step to cover a larger region of the classification boundary, thus capturing more information about boundary difference. In contrast, when the classification boundary is not overlapped, a larger search step will contain a larger region where there is no difference, thus reducing detection sensitivity. This property also implies that it is valid to use different values of $k$ to monitor retreat and invasion. This could be the case when one demands different priorities for retreat and invasion, or a different detection sensitivity for them.

\begin{center}
\begin{table}
\caption{Comparison of concept drift detection accuracy (detection rate/false alarm rate) between the proposed method (NSD) and the competence model-based method (CM). \label{tab:ExpSynthetic}}

\centering{}%
\begin{tabular}{ccccc}
\toprule
\textbf{No.} & \textbf{Drift Type} & \textbf{NSD (k=1)} & \textbf{NSD (k=5)} & \textbf{CM}\tabularnewline
\midrule
1 & Linear Shift 2D (0.03) & \textbf{99/3} & 98/4 & 11/3\tabularnewline
2 & Linear Shift 2D (0.02) & \textbf{98/5} & 71/4 & 10/7\tabularnewline
3 & Linear Shift 2D (0.01) & \textbf{56/4} & 20/4 & 6/7\tabularnewline
4 & Linear Shift 4D (0.03) & \textbf{88/3} & 58/0 & 23/5\tabularnewline
5 & Linear Shift 10D (0.03) & \textbf{43/1} & 23/1 & 7/7\tabularnewline
6 & Linear Rotate 2D $15^{\circ}$ & \textbf{99/7} & 73/3 & 94/8\tabularnewline
7 & Linear Rotate 2D $10^{\circ}$ & \textbf{98/2} & 23/4 & 58/6\tabularnewline
8 & Normal Shift 2D 0.7 & 93/3 & \textbf{99/3} & 87/5\tabularnewline
9 & Normal Shift 2D 0.5 & 64/3 & \textbf{90/5} & 46/3\tabularnewline
\bottomrule
\end{tabular}
\end{table}
\par\end{center}
\vspace{-0.7cm}
\subsubsection*{Experiment 10. Efficiency}

In this experiment, we extend the previous experiment to measure and compare the running times of our method and the competence model-based method (CM) \cite{lu14cmp_c}. The datasets used are uniform(0,1) hyper-cube with a linear shifting classification boundary, similar to the experiment above. Different dimensions and window sizes are used to demonstrate their impact on both algorithms.

The running times are obtained in a server environment with Intel Xeon 2.40GHz CPU, 256GB memory and 64bit Red Hat Linux Operating System. The programs are implemented in Python 2.7 with numpy, scipy library stack. We deliberately avoid parallel computing and use only single threaded implementation for easier performance analysis.

The result is shown in Table \ref{tab:ExpEfficiency}. As can be seen, the overall performance of our method is approximately three magnitudes better (3000 times) than CM. The major reason for this is that we do not need a permutation procedure since the significance level of our statistic is directly calculated from the formula. For CM, however, permutation is mandatory for estimating the significance level, which introduces an overhead of repeatedly computing the statistic hundreds of times (500 in this experiment as recommended by the author). Being aware that this overhead can be alleviated to a certain extent by parallel computing, we also compare our method to CM without counting their permutation cost and only measure the time to compute their statistic once (CM w/o Perm). The results show that the computation of our statistic is still 5\textasciitilde 10 times faster than that of CM. This indicates that our method is much more efficient and applicable for real-time stream learning. And our algorithm is barely affected by the number of dimensions and the number of neighbors, but is sensitive to window size.

\begin{table}
\caption{Comparison of the running times (in seconds, unless indicated otherwise) of different dimensions and window sizes for the proposed method (NSD) and the competence model-based method (CM).\label{tab:ExpEfficiency}}
\centering
\begin{tabular}{cccccc}
\toprule
\textbf{Dim} & \textbf{Size} & \textbf{NSD (1)} & \textbf{NSD (5)} & \textbf{CM} & \textbf{CM w/o Perm}\tabularnewline
\midrule
4 & 5000 & 1.15 & 1.23 & 1h3ms & 7.62\tabularnewline
6 & 5000 & 1.36 & 1.25 & 1h5m & 7.91\tabularnewline
10 & 5000 & 1.53 & 1.6 & 1h16m & 9.21\tabularnewline
10 & 10000 & 5.72 & 6.03 & 5h57m & 42.84\tabularnewline
10 & 15000 & 12.7 & 13.25 & 12h20m & 88.85\tabularnewline
10 & 20000 & 22.41 & 23.19 & 21h43m & 156.39\tabularnewline
\bottomrule
\end{tabular}
\end{table}

\subsubsection*{Experiment 11. Classification task on real-world datasets}

In this experiment, we apply the proposed concept drift method in typical machine learning classification tasks on real-world stream (simulated) data to demonstrate the performance of our method against state-of-the-art drift adaptation methods.

The five datasets we use are top referenced benchmarks for concept drift problems which provide great variety in dimension (10\textasciitilde 500), data types (numerical, categorical, binary) and data size (1K\textasciitilde 50K). The \textbf{Airline dataset} \cite{bifet2010moa} consists of arrival and departure information for commercial flights of US airports, from October 1987 to April 2008. Each data record has 7 features and the classification task is to predict if the flight is delayed or not delayed. The original dataset includes a total of 539,388 records of which 50,000 is used in our experiment for efficiency. The \textbf{Spam Filtering dataset} \cite{katakis2008spam} includes 9,324 email messages collected from the Spam Assassin system (\href{https://spamassassin.apache.org/}{https://spamassassin.apache.org/}). The dataset consists of 500 attributes which are extracted from the original 39,916 attributes using Chi-square feature selection technique. The classification task is to predict spam emails (around 20\%) from legitimate emails. This dataset typically includes gradual concept drift. The \textbf{Usenet1 and Usenet2} \cite{katakis2008usenet} are two datasets from the UCI Machine Learning Repository. Both include 1,500 instances collected from the Twenty Newsgroups. Each instance consists of 99 features representing a message from one of three topics: medicine, space and baseball, and is labeled by a user as interesting or not interesting. Drift in user interest is introduced every 300 messages. Finally, the \textbf{Weather dataset} \cite{elwell11increm} contains weather data from Offutt Air Force Base in Bellevue, Nebraska over 50 years. Daily measurements were taken for a variety of features such as temperature, pressure, visibility, and wind speed. Eight features are used and the classification task is to predict whether rain precipitation was observed on each day. The dataset contains 18,159 instances with 31\% positive (rain) classes, and 69\% negative (no-rain) classes. This dataset contains not only cyclical seasonal changes, but also possible long-term climate change.

\begin{table}
\scriptsize
\caption{Parameters of concept drift handling methods.} \label{tab:ExpMethodParameter}

\begin{tabular}{p{0.05\textwidth}p{0.065\textwidth}p{0.3\textwidth}}
\toprule
\textbf{Category} & \textbf{Method} & \textbf{Parameters}\tabularnewline
\midrule
Our &NSD & $k=1$, window size=50\tabularnewline

\midrule
Drift &HDDM-A & warning level=0.005, drift level=0.001, two-sided\tabularnewline
Detection &HDDM-W & warning level=0.005, drift level=0.001, lambda=0.05, one-sided\tabularnewline
&ADWIN & delta=0.002\tabularnewline
&CM & permutation=500, window size=50\tabularnewline

\midrule
Drift &SAMkNN & STM size=50, LTM size=20\tabularnewline
Adaptation&HAT & split confidence=0, tie threshold=0.05\tabularnewline
 &AUE2 & member count=10, chunk size=50\tabularnewline
&Learn++NSE & period=50, sigmoid slope=0.5, sigmoid crossing point=10, ensemble size=15\tabularnewline
&ARF & ensmble size = 15, window size = 50\tabularnewline
&SRP & ensmble size = 15, window size = 50\tabularnewline
&AMF & ensmble size = 15, window size = 50\tabularnewline
&IWE & ensmble size = 15, window size = 50\tabularnewline
\bottomrule
\end{tabular}
\end{table}

\begin{table*}
\begin{centering}
\caption{Classification accuracy of concept drift detection methods on real-world datasets with different base learners. The rank of the accuracy on each dataset is shown in brackets. The average rank indicates the overall performance of each algorithm (lower is better).\label{tab:ExpReal}}
\setlength{\tabcolsep}{1.2mm}{
\begin{tabular}{ccccccccccc}
\toprule
\multirow{1}{*}{Algorithms} & Learner & \textbf{Airline} & \textbf{Spam} & \textbf{Usenet1} & \textbf{Usenet2} & \textbf{Weather} & Avg. Score & Avg. Score (all) & Avg. Rank & Avg. Rank (all) \\
\midrule

\multirow{3}{*}{NSD} & NB & 69.55 (2) & 91.62 (7) & 74.93 (4) & 72.93 (3) & 69.21 (23) & 75.64 (2) & \multirow{3}{*}{\textbf{75.20 (1)}} & \textbf{7.8 (1)} & \multirow{3}{*}{\textbf{9.3 (1)}} \\

& kNN & \textbf{69.65 (1)} & 90.78 (14) & 67.80 (14) & 68.67 (17) & 75.81 (6) & 74.60 (9) & & 10.4 (5)  & \\

& HTree & 66.66 (16) & 91.15 (10) & \textbf{75.20 (1)} & 72.40 (5) & 71.40 (17) & 75.36 (4) & & 9.8 (4) &\\
\midrule

\multirow{3}{*}{HDDM-A} & NB & 68.96 (4) & 90.94 (12) & 75.18 (2) & 71.00 (11) & 72.34 (14) & \textbf{75.68 (1)} & \multirow{3}{*}{74.99 (2)} & 8.6 (2) & \multirow{3}{*}{9.9 (2)}\\

& kNN & 69.56 (6) & 89.91 (18) & 69.13 (11) & 67.87 (18) & 76.41 (4) & 74.37 (11) & & 11.4 (7) & \\

& HTree & 67.55 (9) & 91.92 (3) & 74.60 (5) & 68.87 (16) & 71.61 (16) & 74.91 (8) & & 9.8 (4) & \\
\midrule

\multirow{3}{*}{HDDM-W} & NB & 67.19 (11) & 91.52 (9) & 75.07 (3) & 70.93 (12) & 72.76 (12) & 75.49 (3) & \multirow{3}{*}{74.74 (3)} & 9.4 (3) & \multirow{3}{*}{11.4 (4)}\\

& kNN & 66.89 (14) & 90.19 (16) & 68.2 (13) & 67.67 (19) & 75.62 (7) & 73.71 (13) & & 13.8 (14) & \\

& HTree & 66.56 (17) & 91.76 (5) & 74.37 (6) & 70.00 (14) & 72.37 (13) & 75.01 (6) & & 11 (6) & \\
\midrule

\multirow{3}{*}{ADWIN} & NB & 68.69 (5) & 91.90 (4) & 67.00 (15) & 72.27 (6) & 70.31 (22) & 74.03 (12) & \multirow{3}{*}{73.32 (6)} & 10.4 (5) & \multirow{3}{*}{11.3 (3)} \\

& kNN & 69.35 (3) & 91.64 (6) & 58.53 (23) & 67.13 (21) & 76.31 (5) & 72.59 (17) & & 11.6 (8) & \\

& HTree & 67.78 (7) & 91.58 (8) & 64.47 (17) & 72.00 (7) & 70.89 (21) & 70.89 (14) & & 12 (9) & \\
 \midrule
 
\multirow{3}{*}{CM} & NB & 66.41 (18) & 91.05 (11) & 70.90 (8) & \textbf{76.91 (1)} & 70.97 (20) & 75.24 (5) & \multirow{3}{*}{74.47 (5)} & 11.6 (8) & \multirow{3}{*}{12.5 (6)} \\

& kNN & 67.28 (10) & 90.11 (17) & 65.33 (16) & 71.30 (9) & 71.88 (15) & 73.18 (15) & & 13.4 (12) & \\

& HTree & 66.12 (20) & 90.72 (15) & 71.24 (7) & 75.87 (2) & 71.06 (19) & 75.00 (7) & & 12.6 (10) & \\
 \midrule
 
\multirow{3}{*}{Learn++NSE} & NB & 63.60 (23) & 69.32 (24) & 70.13 (10) & 65.00 (25) & 68.40 (25) & 67.29 (25) & \multirow{3}{*}{68.33 (13)} & 21.4 (20) & \multirow{3}{*}{18.4 (12)}\\

& kNN & 67.75 (8) & 72.17 (22) & 64.00 (19) & 67.13 (21) & 73.47 (10) & 68.90 (22) & & 16 (17) & \\

& HTree & 66.18 (19) & 68.57 (25) & 70.20 (9) & 70.67 (13) & 68.43 (24) & 68.81 (23) & & 18 (19) & \\
\midrule

SAMkNN & kNN & 65.06 (22) & \textbf{96.26 (1)} & 64.47 (17) & 71.07 (10) & 75.54 (8) & 74.48 (10) & 74.48 (4) & 11.6 (8) & 11.6 (5) \\
\midrule

AUE2 & HTree & 66.67 (15) & 72.16 (23) & 68.40 (12) & 72.70 (4) & 73.78 (9) & 70.74 (21) & 70.74 (11) & 12.6 (10) & 12.6 (7) \\

HAT & HTree & 65.65 (21) & 90.81 (13) & 63.93 (20) & 71.87 (8) & 73.19 (11) & 73.09 (16) & 73.09 (7) & 14.6 (15) & 14.6 (9)\\

ARF & HTree & 66.90 (13) & 89.78 (20) & 59.66 (21) & 67.31 (20) & 77.99 (3) & 72.32 (18) & 72.32 (8) & 15.4 (16) & 15.4 (10) \\

SRP & HTree & 67.18 (12) & 89.85 (19) & 53.45 (24) & 66.28 (24) & 78.46 (2) & 71.04 (20) & 71.04 (10) & 16.2 (18) & 16.2 (11) \\

AMF & HTree & 59.10 (25) & 94.35 (2) & 59.66 (21) & 69.24 (15) & \textbf{78.66 (1)} & 72.20 (19) & 72.20 (9) & 12.8 (11) & 12.8 (8) \\

IWE & HTree & 62.80 (24) & 88.82 (21) & 53.10 (25) & 66.69 (23) & 71.34 (18) & 68.55 (24) & 68.55 (12) &  22.2 (21) & 22.2 (13) \\

\bottomrule
\end{tabular}}
\par\end{centering}
\end{table*}

We compare our method with 12 representative works of different strategies of concept drift handling, including three error rate monitoring-based methods - \textbf{HDDM-A}, \textbf{HDDM-W} \cite{frias2015hddm}, \textbf{ADWIN} \cite{bifet2007adwin}; two instance-based methods - \textbf{CM} \cite{lu14cmp_c}, \textbf{SAMkNN} \cite{losing2016samknn}; one adaptive decision tree method - \textbf{HAT} \cite{bifet2009hat}; and six ensemble methods - \textbf{AUE2} \cite{brzezinski14aue2}, \textbf{Learn++NSE} \cite{elwell11increm}, \textbf{Adaptive Random Forest (ARF)} \cite{Gomes2017Adaptive}, \textbf{Streaming Random Patches (SRP)} \cite{gomes2019streaming}, \textbf{Aggregated Mondrian Forest Classifier (AMF)} \cite{mourtada2021amf}, and \textbf{Incremental Weighted Ensemble (IWE)} \cite{jiao2022incremental}. For the proposed method (NSD), and the other detection methods HDDM-A, HDDM-W, ADWIN and CM, a test-and-retrain strategy is used in order to evaluate their final classification accuracy. This is not required for SAMkNN, HAT and the ensemble methods since they use their own mechanisms to maintain classifiers. Also, most of these methods, including ours, can be used in combination with different base learners. Therefore, for a well-rounded evaluation, when applicable, the algorithms are implemented with three learners, namely \textbf{kNN}, \textbf{Na\"ive Bayes (NB)} and \textbf{Hoeffding Tree (HTree)} classifiers. The exceptions are SAMkNN, which is only compatible with kNN, as well as AUE2, HAT, ARF, SRP, AMF, and IWE, which are packaged with the Hoeffding Tree.

The algorithms are implemented based on the MOA platform \cite{bifet2010moa} using Java. And some ensemble-based drift adaptation methods are implemented by River \cite{montiel2021river} using Python 3.8. The parameters settings are listed in Table \ref{tab:ExpMethodParameter}. For a fair comparison, unless indicated otherwise, the window size is set to 50 as default; the warning level of all tests is set to 0.05 and the drift level is 0.01; the parameter for the kNN learners is set to k=10; the ensemble size is set to 15 as default; and the distance function is set to Euclidean if needed.

The results are listed in Table \ref{tab:ExpReal}. Since no single algorithm outperforms all the others in all cases, we use a ranking method similar to that in \cite{losing2016samknn} to better demonstrate the overall improvements made by the proposed method. First, the classification accuracy of all the methods is ranked for each dataset. After this, the average ranking of each method is calculated and forms the final ranking, which indicates the overall performance of each method in all scenarios. From the results, we see that the proposed method has the best overall performance and the top accuracy on two individual datasets. Additionally, based on an closer analysis, the following observations can be made: 1) the error rate-based methods (HDDM-A/W, ADWIN) achieve relatively good results compared to most methods of other strategies, which manifests their advantages in practice, considering their computational simplicity and efficiency; 2) the instance-based methods (NSD, SAMkNN and CM) achieve best performance on four of the five datasets. This shows that concept drift detection based on instances is sensitive in different scenarios and is a promising strategy to tackle the problem; 3) the choice of base learners plays an important role in the final classification accuracy. For example, for every algorithm, the kNN learner always has the best result on the Weather dataset but the worst result on the Usenet datasets. This indicates that, in addition to an accurate detection method, choosing the right learner is also a consideration of great importance in handling concept drift problems.

\section{Conclusion and future studies\label{sec:Conclusion}}
Traditional error rate-based methods are designed to detect real concept drift but fall short of describing the drift, thus cannot be used to update existing models, whereas distribution test-based methods inherently cannot avoid detecting virtual drift and are often computationally expensive for real-time online tasks. In this work, we overcame these limitations by adopting a new perspective on the nearest neighbor problem and introduced a novel concept, \emph{nearest neighbor searching}. Based on this concept, we developed a series of statistics, including \emph{neighbor-searching volume}, \emph{neighbor-searching volume ratio} and the core measure - \emph{neighbor-searching discrepancy} (NSD). We provided both theoretical analysis and empirical verification to demonstrate the robustness of this measure in various application scenarios, as well as its appealing properties such as distribution-free, unrelated to starting point location and neighbor-searching volume size, and not requiring independence for the joint statistics.

An efficient real concept drift detection method was developed based on NSD. The proposed method: 1) focuses on detecting classification boundary change (real drift) and ignores virtual drift that does not affect the classification boundary. As a result, higher test power is gained compared to distribution tests. This also means that our method is able to detect real drift without relying on particular classifiers, unlike error-based methods; 2) is computationally efficient, since the significance threshold of the underlying statistic can be calculated directly from the formula, without relying on re-sampling methods such as permutation or bootstrap. This makes our method applicable for real-time tasks; 3) is able to describe real concept drift as classification boundary retreat or invasion of each class, which is an appealing feature for understanding concept drift. Based on this information, a classification boundary change can also be described from other perspectives, such as the widening/narrowing of the classification gap or the separability of the two classes.

Our next attempt will aim to develop concept drift adaptation models based on the output of the proposed detection method and utilize the novel information of retreat and invasion. Another possible improvement may be achieved by identifying and dividing the classification boundary into different sections according to the data distribution, thus completely avoiding the canceling effect of boundary rotation. Lastly, because the concepts and statistics proposed in this work are general and robust, other classification-related methods, especially KNN-based methods, may take advantage of these results and improve their performance from new perspectives.

\section*{Acknowledgments}
The work presented in this paper was supported by the Australian Research Council (ARC) under Grant FL190100149. We would like to thank Drs Anjin Liu, Feng Liu, Yiliao Song and Dan Shang for their advice on this work. We also wish to thank the anonymous reviewers for their helpful comments.

\appendices
\section*{\centerline{Appendix A}}

\subsection*{Proof of Lemma 1}

\begin{proof}
Since the probability density function (PDF) is derived as
\begin{equation}
\mathbb{P}'(V(k)\le v)=\mathrm{binom}(k,n,\frac{\lambda v}{n})\cdot\frac{k}{v}\label{eq:VolumePDFbinom}.
\end{equation}
Notice that when $n\gg k$, $\underset{n\rightarrow\infty}{\lim}\frac{k}{n}=0$, thus $\underset{n\rightarrow\infty}{\lim}\frac{n!}{(n-k)!n^{k}}=1$ and $\underset{n\rightarrow\infty}{\lim}(1-\frac{\lambda v}{n})^{n-k}=e^{-\lambda v}$. 
This equation can be formulated as
\begin{align*}
f(v) & =\frac{n!}{(n-k)!n^{k}}\Big(1-\frac{\lambda v}{n}\Big)^{n-k}\cdot\frac{\lambda^{k}v^{k-1}}{(k-1)!}\\
 & =e^{-\lambda v}\cdot\frac{\lambda^{k}v^{k-1}}{(k-1)!},
\end{align*}
which is a Gamma distribution with shape parameter $k$ and scale parameter $\lambda$.
\end{proof}

\subsection*{Proof of Lemma 2}

\begin{proof}
Given two independent random variables $X\sim\mathcal{V}(k_{1},n,\lambda)$, $Y\sim\mathcal{V}(k_{2},n,\lambda)$, $k_{1}>0$, $k_{2}>0$, $\lambda>0$. If $n\gg k_{1}$ and $n\gg k_{2}$, according to Lemma 1, $X\sim\Gamma(k_{1},\lambda),Y\sim\Gamma(k_{2},\lambda)$. Let $U=X+Y$ and $V=X/(X+Y)$. The joint probability of $(X,Y)$ has PDF
\begin{align*}
f(x,y) & =\frac{\lambda^{k_{1}}}{\Gamma(k_{1})}x^{k_{1}-1}e^{-\lambda x}\frac{\lambda^{k_{2}}}{\Gamma(k_{2})}y^{k_{2}-1}e^{-\lambda y}\\
& =\frac{\lambda^{k_{1}+k_{2}}}{\Gamma(k_{1})\Gamma(k_{2})}x^{k_{1}-1}y^{k_{2}-1}e^{-\lambda(x+y)}.
\end{align*}
Since $x=uv$, $y=u(1-v)$, the absolute value of the Jacobian of $(U,V)$ and $(X,Y)$ is
\[
\left|\det\frac{\partial(x,y)}{\partial(u,v)}\right|=u,
\]
the PDF of $(U,V)$ is thus
\begin{align*}
g(u,v) & =\frac{\lambda^{k_{1}+k_{2}}}{\Gamma(k_{1})\Gamma(k_{2})}(uv)^{k_{1}-1}[u(1-v)]^{k_{2}-1}e^{-\lambda u}u\\
 & =\frac{\lambda^{k_{1}+k_{2}}}{\Gamma(k_{1})\Gamma(k_{2})}u^{k_{1}+k_{2}-1}e^{-\lambda u}v^{k_{1}-1}(1-v)^{k_{2}-1}\\
 & =\frac{\lambda^{k_{1}+k_{2}}}{\Gamma(k_{1}+k_{2})}u^{k_{1}+k_{2}-1}e^{-\lambda u}\frac{\Gamma(k_{1}+k_{2})}{\Gamma(k_{1})\Gamma(k_{2})}v^{k_{1}-1}(1-v)^{k_{2}-1}.
\end{align*}
According to the factorization theorem, $v=x/(x+y)=R_{(k_{1},k_{2})}$ has the PDF $\mathrm{Beta}(k_{1},k_{2})=\frac{\Gamma(k_{1}+k_{2})}{\Gamma(k_{1})\Gamma(k_{2})}v^{k_{1}-1}(1-v)^{k_{2}-1}$.
\end{proof}

\section*{\centerline{Appendix B}}
\subsection*{Proof of Theorem 1}
~~~~~~The basic idea for us to prove Theorem 1 is that we seek a good transformation $T$ to map distribution ${P}$ into a manifold space, where the new distribution $T_{\#}P$ becomes an uniform distribution. Then we compute $\mathbb{P}(R_{(k_{1},k_{2})}<0.5)$ in the new space. Finally, we prove the value of $\mathbb{P}(R_{(k_{1},k_{2})}<0.5)$ in the new space is same with the value of $\mathbb{P}(R_{(k_{1},k_{2})}<0.5)$ in $\mathbb{R}^d$.

Before proving our main result (Theorem 1), we need to extend the definitions of neighbour searching and $k$th nearest neighbor to the Riemannian manifold space $M^d$.

We use Riemannian manifold $M^d$ to replace $\mathbb{R}^d$.
There are two benefits. 1). $\mathbb{R}^d$ is just
a special case of manifold. Therefore, our definition about
nearest neighbor can handle more complicated situations.
2) The proving technique of our main theorem depends
on differential geometry theory. Using manifold to define
nearest neighbors will be helpful for us to prove our main
theorem.
\begin{defn}[Neighbor Searching]
Given a Riemannian manifold $M^d$, define a\emph{ (neighbor) search} ${S}$ over $M^d$ to be a sequence of measurable sets $\{S_{i}\}$, called \emph{search steps}, with the following properties:

1) $S_{0}=\{x^{1},x^{2},\ldots,x^{m}\}$, where $x^i \in M^d$, is the initial search step containing $m$ \emph{starting points};

2) $S_{i}\subset S_{j}$ if $i<j$, where $i,j\in\mathcal{I}$;

3) for any $i<j,i,j\in\mathcal{I}$, we have $V(S_{i})<V(S_{j})<V(M^d)$;

4) $\sup_{i\in \mathcal{I}} V(S_i))=V(M^d)$;

5) For any $v \in [0,V(M^d))$, there exists $i \in \mathcal{I}$ such that $V(S_i)=v$;

where $\mathcal{I}=[0,+\infty)$ is the index set.
\label{def:NeighborSearch}
\end{defn}

\begin{defn}[$k$th Nearest Neighbor]
  Given a Riemannian manifold $M^d$, continuous distribution $P$ neighbor searching $S$, and samples $X=\{x_1,...,x_n\}$ distributed by $P$, we say a set $S{(k)}\in S$ is the $k$th (nearest) neighbor if
  
  1) there exists an index $i\in \mathcal{I}$ such that $S{(k)}=S_i$;

  2) the number of samples in $S_i$ is $k$ ($\# X \bigcap S_i=k$);

  3) the number of samples in $S_j$ is small than k ($\# X \bigcap S_j<k$), where $j<i$ and $j\in \mathcal{I}$.
  
   The $k$th (nearest) neighbor-searching volume, denoted as $V(k) = V (S(k))$, is the volume of the $k$th nearest neighbor set $S(k)$.
\label{def:NearestNeighbor}
\end{defn}

\begin{defn}[Neighbor Search Restricted in Subset]
Given an open set $A \subset M^d$ and neighbor search $S=\{S_i\}_{i\in \mathcal{I}}$, a family of set $S_A$ is called neighbor search restricted in $A$, if $S_A=\{S_i\bigcap A\}_{i\in \mathcal{I}}$ .
\end{defn}

\begin{prop}\label{0.}
Given a manifold $M^d$, an open set  $A \subset M^d$ and  neighbor search $S=\{S_i\}_{i\in \mathcal{I}}$, then for neighbor search restricted in $A$: $S_A=\{S_i\bigcap A\}_{i\in \mathcal{I}}$,  any  $v \in [0,\sup_{i\in \mathcal{I}} V(S_i\bigcap A))$, there exists $i \in \mathcal{I}$ such that $V(S_i\bigcap A)=v$.
\end{prop}

\begin{proof}
1)  If $v=0$, then we choose $i=0$.

2)  If $0<v<\sup_{i\in \mathcal{I}} V(S_i\bigcap A)$, because $V(S_0\bigcap A)=0$ and $\lim_{i\rightarrow \infty}V(S_i\bigcap A)=\sup_{i\in \mathcal{I}} V(S_i\bigcap A)$, then there exist $\alpha$ and $\beta$ such that

\begin{equation}
V(S_{\alpha}\bigcap A)<v<V(S_{\beta}\bigcap A),
\end{equation}

Then we choose $\rho_0=\frac{\alpha+\beta}{2}$, if $V(S_{\rho_0}\bigcap A)\leq v$, we set $\alpha_0=\rho_0$; otherwise, $\beta_0=\rho_0$. We define $\rho_i=\frac{\alpha_i+\beta_i}{2}$, $i\in \mathbb{Z}_{\geq 0}$. If $V(S_{\rho_i}\bigcap A)\leq v$, we set $\alpha_{i}=\rho_i$; otherwise, $\beta_i=\rho_i$. Then we obtain  interval sequences
$\{[\alpha_i,\beta_i]\}_{i=0}^{\infty}$ such that

1) $\lim_{i\rightarrow \infty}\beta_i-\alpha_i=0$;

2) $\alpha_{i}\leq \alpha_{i+1}$ and $\beta_{i+1}\leq \beta_{i}$.
Then according to Closed Interval Theorem, there exists $t$ such that $\lim_{i\rightarrow \infty}\beta_i=\lim_{i\rightarrow \infty} \alpha_i=t$. 

Then we claim that $V(S_t \bigcap A)=v$. 
Because of properties 3) and 5) in Definition \ref{def:NeighborSearch}, we can proof that 

\begin{equation}\label{1.}
\lim_{i\rightarrow t} V(S_i)=V(S_t).
\end{equation}

Then

1)
\begin{equation*}
\lim_{i\rightarrow \infty}V(S_{\alpha_i}\bigcap A)\leq v \leq \lim_{i\rightarrow \infty}V(S_{\beta_i}\bigcap A);
\end{equation*}

2)
\begin{equation*}
\lim_{i\rightarrow \infty}V(S_{\alpha_i}\bigcap A)\leq V(S_{t}\bigcap A) \leq \lim_{i\rightarrow \infty}V(S_{\beta_i}\bigcap A);
\end{equation*}

3)
\begin{equation*}
V(S_{\beta_i}-S_{\alpha_i})\geq
V(S_{\beta_i}\bigcap A) - V(S_{\alpha_i}\bigcap A)\geq 0;
\end{equation*}

4) According to formula $(\ref{1.})$,
\begin{equation*}
\lim_{i\rightarrow \infty}V(S_{\beta_i}-S_{\alpha_i})=0.
\end{equation*}

1), 2), 3) and 4) imply that  $V(S_{t}\bigcap A)=v$.
\end{proof}

Given a continuous distribution in manifold $M^d$:
\begin{equation}
P(U)=\int_{M^d\bigcap U}f dV,
\end{equation}
where $U$ is any Borel set and $f$ is the continuos function in $M^d$.

\begin{defn}[$k$th Nearest Neighbor Restricted in Subset]
Given manifold $M^d$, continuous distribution $P$ neighbor search $S$, open set $A \subset M^d$ and samples $X=\{x_1,...,x_n\}$ distributed by $P$, we say a set $S_A{(k)}\in S_A$ is the $k$th (nearest) neighbor if

1) there exists an index $i\in \mathcal{I}$ such that $S_A{(k)}=S_i\bigcap A$;

2) the number of samples in $S_i\bigcap A$ is k ($\# X \bigcap (S_i\bigcap A)=k$);

3) the number of samples in $S_j$ is small than k ($\# X \bigcap ( S_j\bigcap A)<k$), where $j<i$ and $j\in \mathcal{I}$.

The $k$th (nearest) neighbor search volume, denoted as $V_A(k) = V (S_A(k))$, is the volume of the $k$th nearest neighbor set $S_A(k)$.
\end{defn}

Given samples $X_1=\{x^1_1,...,x^1_{n_1}\}$ and $X_2=\{x^1_1,...,x^1_{n_2}\}$ which are distributed by $P$. We denote $V({k_1})$ as the volume $k_1$th nearest neighbor  according to samples $X_1$ and $V({k_2})$ as the volume $k_2$th nearest neighbor  according to samples $X_2$. Then we prove that

\begin{prop}\label{4.}
\begin{equation*}
\mathbb{P}(V(k_1)< V(k_2))=\mathbb{P}(V_A(k_1)< V_A(k_2)),
\end{equation*}
where the set $A=\{x\in M^d: f(x)>0\}$.
\end{prop}

\begin{proof}
We only need to prove that
\begin{equation}
V(k_1)< V(k_2) \leftrightarrow  V_A(k_1)< V_A(k_2), a.e.;
\end{equation}
Because 
\begin{equation}
\begin{split}
&~~~~~~V(k_1)< V(k_2)\\& \leftrightarrow \# S(k_1)\bigcap X_2 <k_2\\&\leftrightarrow\# S(k_1)\bigcap (X_2\bigcap A) <k_2, a.e.\\& \leftrightarrow V_A(k_1)< V_A(k_2), a.e.;
\end{split}
\end{equation}
here we use $X_2\subset A, a.e.$
\end{proof}

Now we prove the main theorem of the paper. For proving Theorem 1, we need to introduce some lemmas.

\begin{lem}\label{2.}
If $M^d=\mathbb{R}^d$ and the density function $f$ is a $C^k$ functions in open set $A=\{x\in M^d: f(x)>0\}$ ($k>2$), then there exists a $C^k$ homeomorphism $T$ to make sure that

1) the image $T(A)$ is a $C^k$ manifold $N^d$;

2) $T_{\#}(P)(U)=\int_{N^d\bigcap U}1dV$, where $U$ is any Borel set in $N^d$.
\end{lem}

This lemma tells us that for any $C^k$ distribution ($k>2$), there exists a $``good"$ transformation $T$ to map the distribution into an uniform distribution in a new space $N^d$.

\begin{proof}
In the open set $A$, for any Borel set $U$,
\begin{equation}
P(U)=\int_{A\bigcap U} f dx;
\end{equation}
First, we assume there exists such map $T$ to make sure $T_{\#}(P)(T(U))=\int_{N^d\bigcap T(U)}1dV$,
then we try to find what will happen.

1)
\begin{equation}
P(U)=T_{\#}(P)(T(U)),
\end{equation}
which implies that

2)
\begin{equation}
\int_{A\bigcap U} f dx=\int_{N^d\bigcap T(U)}1dV
\end{equation}
We use area formula and obtain that
\begin{equation}
\int_{A\bigcap U} f dx=\int_{N^d\bigcap T(U)}1dV=\int_{A\bigcap U}\sqrt{G}dx,
\end{equation}
where $G$ is $\left|det[\frac{\partial T}{\partial x}]\right|^2$.
Hence, 
\begin{equation}
f^2=G.
\end{equation}
Therefor, for any homeomorphism $T$, if $T$ satisfies $f^2=G$, then the density function of the new distribution after distribution $T$ is 1.

We use the Nash embedding theorem \cite{nash1956imbedding} to help us find such transformation $T$ which satisfies $f^2=G$.

In $A$, we need to define a Riemannian metric $[g_{ij}]=f^{\frac{2}{d}}I_{d\times d}$, where $I_{d\times d}$ is the $d\times d$ identity matrix.
Then $(A,[g_{ij}])$ is an abstract $C^k$ Riemannian manifold. According to Nash embedding theorem, we know that there exists a $C^k$ homeomorphism $T$ embedding $A$ into a new space $\mathbb{R}^K$  ($K\leq d^2+5d+3$), then $T(A)$ is the new manifold $N^d$ we want to find.
\end{proof}

Until now, according to proposition \ref{0.} and lemma \ref{2.}, we can prove 

\begin{lem}
Given $C^k$ denstity f in $A\subset \mathbb{R}^d$ and $C^k$ homeomorphism $T$ ($T: A\rightarrow N^d$), we define a volume map 
\begin{equation}
T_f: [0,\sup_{i\in \mathcal{I}} \mathcal{L}^d(S_i\bigcap A))\rightarrow [0,\sup_{i\in \mathcal{I}} V(T(S_i\bigcap A))),
\end{equation}
for any $v\in [0,\sup_{i\in \mathcal{I}} \mathcal{L}^d(S_i\bigcap A))$, we choose $S_k\in S$ such that $\mathcal{L}^d(S_k\bigcap A)=v$, and set $T_f(v)=V(T(S_k\bigcap A)))$, then the map $T_f$ is well defined, monotone bijective map.
\end{lem}

Now we can use lemma \ref{2.} to help us compute $\mathbb{P}(V_A(k)>v)$.

\begin{thm}\label{6.}
Given $C^k$ denstity f in $A\subset \mathbb{R}^d$ and $C^k$ homeomorphism $T$ ($T: A\rightarrow N^d$),
\begin{equation}
\mathbb{P}(V_A(k)>v)=\sum_{i=0}^{k-1} \binom{n}{i}\Big(T_f(v)\Big)^{i}\Big(1-T_f(v)\Big)^{n-i},
\end{equation}
where $v\in  [0,\sup_{i\in \mathcal{I}} \mathcal{L}^d(S_i\bigcap A))$.
\end{thm}

\begin{proof}
According to lemma \ref{2.}, there exists a neighbor set $S_i$ such that $\mathcal{L}^d(S_i\bigcap A)=v$, then
\begin{equation}
\begin{split}
\mathbb{P}(V_A(k)>v)&=\mathbb{P}(\# (S_i\bigcap A) \bigcap X\leq k-1)\\&=\mathbb{P}(\# T(S_i\bigcap A) \bigcap T(X)\leq k-1)
\\&=\sum_{i=0}^{k-1}\mathbb{P}(\# T(S_i\bigcap A) \bigcap T(X)=i)\\&=\sum_{i=0}^{k-1} \binom{n}{i}\Big(p\Big)^{i}\Big(1-p\Big)^{n-i},
\end{split}
\end{equation}
where $p=P(S_i\bigcap A)=\int_{N^d\bigcap T(S_i)}dV=V_{N^d}(T(S_i))=T_f(v)$.
Then
\begin{equation}
\mathbb{P}(V_A(k)>v)=\sum_{i=0}^{k-1} \binom{n}{i}\Big(T_f(v)\Big)^{i}\Big(1-T_f(v)\Big)^{n-i}
\end{equation}
\end{proof}

\begin{cor}\label{5.}
Under the same assumption of Theorem \ref{6.}, let $V_{N^d}(k)$ is the  volume in $N^d$ of set $T(S(k)\bigcap A)$, where $S_{k}$ is the $k$th nearest neighbour corresponding to samples $X$, then we have
\begin{equation}
\mathbb{P}(V_{N^d}(k)>v)=\sum_{i=0}^{k-1} \binom{n}{i}\Big( v\Big)^{i}\Big(1- v\Big)^{n-i}
\end{equation}
where $v\in [0,1)$.
Moreover, $V_{N^d}(k)'s$ density function of   $f_{V(k)}(v)$ is
\begin{equation}
\begin{split}
-\sum_{i=1}^{k-1}\binom{n}{i}\Big[iv^{i-1}(1- v)^{n-i}+v^i(n-i)(1- v)^{n-i-1}(-1)\Big]
\end{split}
\end{equation}
when $v\in [0,1)$;
$f_{V(k)}(v)=0$, when $v\geq 1$.
\end{cor}

Now we prove Theorem 1.
\begin{proof}[Proof for Theorem 1]
According to Lemma \ref{2.}, there exists a homeomorphism $T$ such that

1) the image $T(A)$ is a $C^k$ manifold $N^d$;

2) $T_{\#}(P)(U)=\int_{N^d\bigcap U}1dV$, where $U$ is any Borel set in $N^d$.

\textbf{Claim 1}
\begin{equation}
\begin{split}
&~~~~~\mathbb{P}(V(k_1)< V(k_2))\\&=\mathbb{P}(V_A(k_1)< V_A(k_2))\\&=\mathbb{P}(V_{N^d}(k_1)< V_{N^d}(k_2)),
\end{split}
\end{equation}
where $V_{N^d}(k_i)$ is the volume in $N^d$ of set $T(S(k_i)\bigcap A)$, here $S_{k_i}$ is the $k_i$th nearest neighbour corresponding to samples $X_i$.

We prove the Claim.
Firstly according to proposition \ref{4.}, $\mathbb{P}(V(k_1)< V(k_2))=\mathbb{P}(V_A(k_1)< V_A(k_2))$. So we only need to prove $\mathbb{P}(V_A(k_1)< V_A(k_2))=\mathbb{P}(V_{N^d}(k_1)< V_{N^d}(k_2))$

It is clear that
\begin{equation}
\begin{split}
&~~~~~~V_A(k_1)< V_A(k_2)\\&\leftrightarrow\# S(k_1)\bigcap (X_2\bigcap A) <k_2, \\& \leftrightarrow \# T(S(k_1)\bigcap A) \bigcap T(X_2) <k_2, \\&\leftrightarrow V_{N^d}(k_1)< V_{N^d}(k_2);
\end{split}
\end{equation}
Therefore, $\mathbb{P}(V_A(k_1)< V_A(k_2))=\mathbb{P}(V_{N^d}(k_1)< V_{N^d}(k_2))$.

So we just need to compute $\mathbb{P}(V_{N^d}(k_1)< V_{N^d}(k_2))$.

Let the density of $V_{N^d}(k)$ be $f_{V(k)}(v)$, where $v\geq 0$, then
\begin{equation}
\begin{split}
&~~~~~\mathbb{P}(V_{N^d}(k_1)< V_{N^d}(k_2))\\&=\int_{0\leq v_1< v_2}f_{V(k_1)}(v_1)f_{V(k_2)}(v_2)dv_1dv_2,
\end{split}
\end{equation}
where $f_{V(k_1)}, f_{V(k_2)}$ are defined in Corollary 1.
However, 
\begin{equation*}
\int_{0\leq v_1< v_2}f_{V(k_1)}(v_1)f_{V(k_2)}(v_2)dv_1dv_2,
\end{equation*}
just depends on $k_1$, $k_2$ and $n_1$, $n_2$.
\end{proof}
We should note in the proof of Theorem 1, we need the density function $f$ is $C^k$ ($k>2$). However, if $f$ is continuous, according the uniform approximate theorem, we can construct $C^k$ functions $\{f_i\}_{1=1}^{+\infty}$ to approximate $f$ and get the result what we want.  It is a common technique used in mathematical analysis.

\bibliographystyle{IEEEtran}
\bibliography{IEEEabrv,main,2019}

\begin{thebibliography}{10}
\providecommand{\url}[1]{#1}
\csname url@samestyle\endcsname
\providecommand{\newblock}{\relax}
\providecommand{\bibinfo}[2]{#2}
\providecommand{\BIBentrySTDinterwordspacing}{\spaceskip=0pt\relax}
\providecommand{\BIBentryALTinterwordstretchfactor}{4}
\providecommand{\BIBentryALTinterwordspacing}{\spaceskip=\fontdimen2\font plus
\BIBentryALTinterwordstretchfactor\fontdimen3\font minus \fontdimen4\font\relax}
\providecommand{\BIBforeignlanguage}[2]{{%
\expandafter\ifx\csname l@#1\endcsname\relax
\typeout{** WARNING: IEEEtran.bst: No hyphenation pattern has been}%
\typeout{** loaded for the language `#1'. Using the pattern for}%
\typeout{** the default language instead.}%
\else
\language=\csname l@#1\endcsname
\fi
#2}}
\providecommand{\BIBdecl}{\relax}
\BIBdecl

\bibitem{Lu2019Learning}
J.~Lu, A.~Liu, F.~Dong, F.~Gu, J.~Gama, and G.~Zhang, ``Learning under concept drift: {A} review,'' \emph{{IEEE} Transactions on Knowledge and Data Engineering}, vol.~31, no.~12, pp. 2346--2363, 2019.

\bibitem{Liu2022Concept}
A.~Liu, J.~Lu, Y.~Song, J.~Xuan, and G.~Zhang, ``Concept drift detection delay index,'' \emph{IEEE Transactions on Knowledge and Data Engineering}, pp. 1--13, 2022.

\bibitem{liu2020concept}
A.~Liu, J.~Lu, and G.~Zhang, ``Concept drift detection via equal intensity k-means space partitioning,'' \emph{IEEE transactions on cybernetics}, vol.~51, no.~6, pp. 3198--3211, 2020.

\bibitem{hulten01minin.c5}
G.~Hulten, L.~Spencer, and P.~Domingos, ``Mining time-changing data streams,'' in \emph{Proceedings of the Seventh ACM SIGKDD International Conference on Knowledge Discovery and Data Mining - KDD '01}, 2001.

\bibitem{chiang01fault.i1}
L.~H. Chiang, E.~L. Russell, and R.~D. Braatz, \emph{Fault Detection and Diagnosis in Industrial Systems}, ser. Advanced Textbooks in Control and Signal Processing.\hskip 1em plus 0.5em minus 0.4em\relax Springer London, 2001.

\bibitem{wei02turnin.c7}
C.-P. Wei and I.-T. Chiu, ``Turning telecommunications call details to churn prediction: A data mining approach,'' \emph{Expert Systems with Applications}, vol.~23, no.~2, pp. 103--112, 2002.

\bibitem{Yu2022Continuous}
H.~Yu, J.~Lu, and G.~Zhang, ``Continuous support vector regression for nonstationary streaming data,'' \emph{IEEE Transactions on Cybernetics}, vol.~52, no.~5, pp. 3592--3605, 2022.

\bibitem{zhou2023multi}
M.~Zhou, J.~Lu, Y.~Song, and G.~Zhang, ``Multi-stream concept drift self-adaptation using graph neural network,'' \emph{IEEE Transactions on Knowledge and Data Engineering}, 2023.

\bibitem{bach08paired}
S.~H. Bach and M.~A. Maloof, ``Paired learners for concept drift,'' in \emph{2008 Eighth IEEE International Conference on Data Mining}, 12 2008.

\bibitem{alippi13just}
C.~Alippi, G.~Boracchi, and M.~Roveri, ``Just-in-time classifiers for recurrent concepts,'' \emph{IEEE Transactions on Neural Networks and Learning Systems}, vol.~24, no.~4, pp. 620--634, 2013.

\bibitem{gama03accur}
J.~Gama, R.~Rocha, and P.~Medas, ``Accurate decision trees for mining high-speed data streams,'' in \emph{Proceedings of the Ninth ACM SIGKDD International Conference on Knowledge Discovery and Data Mining - KDD '03}, 2003.

\bibitem{yang12increm}
H.~Yang and S.~Fong, ``Incrementally optimized decision tree for noisy big data,'' in \emph{Proceedings of the 1st International Workshop on Big Data, Streams and Heterogeneous Source Mining Algorithms, Systems, Programming Models and Applications - BigMine '12}, 2012, pp. 36--44.

\bibitem{gomes17adapt}
H.~M. Gomes, A.~Bifet, J.~Read, J.~P. Barddal, F.~Enembreck, B.~Pfharinger, G.~Holmes, and T.~Abdessalem, ``Adaptive random forests for evolving data stream classification,'' \emph{Machine Learning}, vol. 106, no. 9-10, pp. 1469--1495, 2017.

\bibitem{jiao2024incremental}
B.~Jiao, Y.~Guo, C.~Yang, J.~Pu, Z.~Zheng, and D.~Gong, ``Incremental weighted ensemble for data streams with concept drift,'' \emph{IEEE Transactions on Artificial Intelligence}, vol.~5, no.~01, pp. 92--103, 2024.

\bibitem{Pratama2020online}
M.~Pratama, E.~Dimla, T.~Tjahjowidodo, W.~Pedrycz, and E.~Lughofer, ``Online tool condition monitoring based on parsimonious ensemble+,'' \emph{IEEE Transactions on Cybernetics}, vol.~50, no.~2, pp. 664--677, 2020.

\bibitem{lu14cmp_c}
N.~Lu, G.~Zhang, and J.~Lu, ``Concept drift detection via competence models,'' \emph{Artificial Intelligence}, vol. 209, pp. 11--28, 2014.

\bibitem{yasum07quick_c41}
Y.~Yasumura, N.~Kitani, and K.~Uehara, ``Quick adaptation to changing concepts by sensitive detection,'' in \emph{International Conference on Industrial, Engineering and Other Applications of Applied Intelligent Systems - IEA/AIE 2007}, H.~G. Okuno and M.~Ali, Eds.\hskip 1em plus 0.5em minus 0.4em\relax Berlin, Heidelberg: Springer Berlin Heidelberg, 2007, pp. 855--864.

\bibitem{li2009concept_c42}
P.~Li, X.~Hu, Q.~Liang, and Y.~Gao, ``Concept drifting detection on noisy streaming data in random ensemble decision trees,'' in \emph{International Workshop on Machine Learning and Data Mining in Pattern Recognition - MLDM 2009}, P.~Perner, Ed.\hskip 1em plus 0.5em minus 0.4em\relax Berlin, Heidelberg: Springer Berlin Heidelberg, 2009, pp. 236--250.

\bibitem{dasu06kdqtree}
T.~Dasu, S.~Krishnan, S.~Venkatasubramanian, and K.~Yi, ``An information-theoretic approach to detecting changes in multi-dimensional data streams,'' in \emph{In Proceedings of the Symposium on the Interface of Statistics, Computing Science, and Applications}, May 2006.

\bibitem{yang2019novel}
Z.~Yang, S.~Al-Dahidi, P.~Baraldi, E.~Zio, and L.~Montelatici, ``A novel concept drift detection method for incremental learning in nonstationary environments,'' \emph{IEEE transactions on neural networks and learning systems}, 2019.

\bibitem{alippi08just_i}
C.~Alippi and M.~Roveri, ``Just-in-time adaptive classifiers part i: Detecting nonstationary changes,'' \emph{IEEE Transactions on Neural Networks}, vol.~19, no.~7, pp. 1145--1153, 2008.

\bibitem{ijcai2018yu}
\BIBentryALTinterwordspacing
S.~Yu, X.~Wang, and J.~C. Príncipe, ``Request-and-reverify: Hierarchical hypothesis testing for concept drift detection with expensive labels,'' in \emph{Proceedings of the Twenty-Seventh International Joint Conference on Artificial Intelligence, {IJCAI-18}}.\hskip 1em plus 0.5em minus 0.4em\relax International Joint Conferences on Artificial Intelligence Organization, 7 2018, pp. 3033--3039. [Online]. Available: \url{https://doi.org/10.24963/ijcai.2018/421}
\BIBentrySTDinterwordspacing

\bibitem{gretton12mmd}
A.~Gretton, K.~M. Borgwardt, M.~J. Rasch, B.~Sch\"{o}lkopf, and A.~Smola, ``A kernel two-sample test,'' \emph{The Journal of Machine Learning Research}, vol.~13, no.~1, pp. 723--773, Mar 2012.

\bibitem{vorburger06entrop}
P.~Vorburger and A.~Bernstein, ``Entropy-based concept shift detection,'' in \emph{Sixth International Conference on Data Mining (ICDM'06)}, 12 2006.

\bibitem{liu11trian}
Z.~Liu and R.~Modarres, ``A triangle test for equality of distribution functions in high dimensions,'' \emph{Journal of Nonparametric Statistics}, vol.~23, no.~3, pp. 605--615, 2011.

\bibitem{biswas14nonpar}
M.~Biswas and A.~K. Ghosh, ``A nonparametric two-sample test applicable to high dimensional data,'' \emph{Journal of Multivariate Analysis}, vol. 123, pp. 160--171, 2014.

\bibitem{liu2018accumulating}
A.~Liu, J.~Lu, F.~Liu, and G.~Zhang, ``Accumulating regional density dissimilarity for concept drift detection in data streams,'' \emph{Pattern Recognition}, vol.~76, pp. 256--272, 2018.

\bibitem{lu16edit}
N.~Lu, J.~Lu, G.~Zhang, and R.~L. de~Mantaras, ``A concept drift-tolerant case-base editing technique,'' \emph{Artificial Intelligence}, vol. 230, pp. 108--133, 2016.

\bibitem{ripley2005spatial}
B.~D. Ripley, \emph{Spatial Statistics}.\hskip 1em plus 0.5em minus 0.4em\relax John Wiley and Sons, 2005.

\bibitem{gama14survey_s}
J.~Gama, I.~{\v{Z}}liobait{\.{e}}, A.~Bifet, M.~Pechenizkiy, and A.~Bouchachia, ``A survey on concept drift adaptation,'' \emph{ACM Computing Surveys}, vol.~46, no.~4, pp. 1--37, 2014.

\bibitem{mardia1970families}
K.~V. Mardia, \emph{Families of Bivariate Distributions}.\hskip 1em plus 0.5em minus 0.4em\relax London: Griffin, 1970.

\bibitem{bifet2010moa}
A.~Bifet, G.~Holmes, R.~Kirkby, and B.~Pfahringer, ``Moa: Massive online analysis,'' \emph{J. Mach. Learn. Res.}, vol.~11, pp. 1601--1604, Aug. 2010.

\bibitem{katakis2008spam}
I.~Katakis, G.~Tsoumakas, E.~Banos, N.~Bassiliades, and I.~Vlahavas, ``An adaptive personalized news dissemination system,'' \emph{Journal of Intelligent Information Systems}, vol.~32, no.~2, pp. 191--212, 2008.

\bibitem{katakis2008usenet}
I.~Katakis, G.~Tsoumakas, and I.~P. Vlahavas, ``An ensemble of classifiers for coping with recurring contexts in data streams,'' in \emph{18th European Conference on Artificial Intelligence}, 2008, Conference Proceedings, pp. 763--764.

\bibitem{elwell11increm}
R.~Elwell and R.~Polikar, ``Incremental learning of concept drift in nonstationary environments,'' \emph{IEEE Transactions on Neural Networks}, vol.~22, no.~10, pp. 1517--1531, 2011.

\bibitem{frias2015hddm}
I.~Frias-Blanco, J.~d. Campo-Avila, G.~Ramos-Jimenez, R.~Morales-Bueno, A.~Ortiz-Diaz, and Y.~Caballero-Mota, ``Online and non-parametric drift detection methods based on hoeffding's bounds,'' \emph{IEEE Transactions on Knowledge and Data Engineering}, vol.~27, no.~3, pp. 810--823, 2015.

\bibitem{bifet2007adwin}
A.~Bifet and R.~Gavald{\`{a}}, ``Learning from time-changing data with adaptive windowing,'' in \emph{Proceedings of the 2007 SIAM International Conference on Data Mining}, vol.~7.\hskip 1em plus 0.5em minus 0.4em\relax SIAM, 2007, Conference Proceedings, p. 2007.

\bibitem{losing2016samknn}
V.~Losing, B.~Hammer, and H.~Wersing, ``Knn classifier with self adjusting memory for heterogeneous concept drift,'' in \emph{Proceedings of the 16th International Conference on Data Mining}, 2016, Conference Proceedings, pp. 291--300.

\bibitem{bifet2009hat}
A.~Bifet and R.~Gavald{\`{a}}, ``Adaptive learning from evolving data streams,'' in \emph{Proceedings of the 8th International Symposium on Intelligent Data Analysis}.\hskip 1em plus 0.5em minus 0.4em\relax Springer, 2009, Conference Proceedings, pp. 249--260.

\bibitem{brzezinski14aue2}
D.~Brzezinski and J.~Stefanowski, ``Reacting to different types of concept drift: The accuracy updated ensemble algorithm,'' \emph{IEEE Tranactions on Neural Networks and Learning Systems}, vol.~25, no.~1, pp. 81--94, 2014.

\bibitem{Gomes2017Adaptive}
H.~M. Gomes, A.~Bifet, J.~Read, J.~P. Barddal, F.~Enembreck, B.~Pfharinger, G.~Holmes, and T.~Abdessalem, ``Adaptive random forests for evolving data stream classification,'' \emph{Machine Learning}, vol. 106, no. 9-10, pp. 1469--1495, 2017.

\bibitem{gomes2019streaming}
H.~M. Gomes, J.~Read, and A.~Bifet, ``Streaming random patches for evolving data stream classification,'' in \emph{2019 IEEE International Conference on Data Mining}.\hskip 1em plus 0.5em minus 0.4em\relax IEEE, 2019, pp. 240--249.

\bibitem{mourtada2021amf}
J.~Mourtada, S.~Ga{\"\i}ffas, and E.~Scornet, ``{AMF}: Aggregated mondrian forests for online learning,'' \emph{Journal of the Royal Statistical Society Series B: Statistical Methodology}, vol.~83, no.~3, pp. 505--533, 2021.

\bibitem{jiao2022incremental}
B.~Jiao, Y.~Guo, C.~Yang, J.~Pu, Z.~Zheng, and D.~Gong, ``Incremental weighted ensemble for data streams with concept drift,'' \emph{IEEE Transactions on Artificial Intelligence}, vol.~5, no.~1, pp. 92--103, 2024.

\bibitem{montiel2021river}
J.~Montiel, M.~Halford, S.~M. Mastelini, G.~Bolmier, R.~Sourty, R.~Vaysse, A.~Zouitine, H.~M. Gomes, J.~Read, T.~Abdessalem \emph{et~al.}, ``River: machine learning for streaming data in python,'' \emph{Journal of Machine Learning Research}, vol.~22, pp. 1--8, 2021.

\bibitem{nash1956imbedding}
J.~Nash, ``The imbedding problem for riemannian manifolds,'' \emph{Annals of mathematics}, pp. 20--63, 1956.

\end{thebibliography}

\begin{IEEEbiography}[{\includegraphics[width=1in,height=1.25in, clip,keepaspectratio]{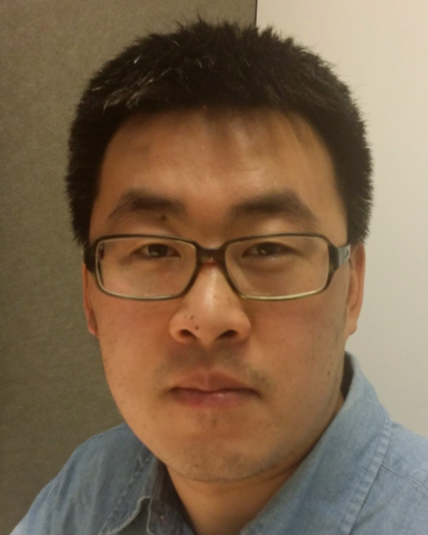}}]{Feng Gu} is currently working as a Software Senior Principal Engineer at Dell Technologies, Shanghai, China. He received his Ph.D. degree in computer science at the University of Technology Sydney in 2020. He received a Bachelor of Software Engineering at Zhejiang University, China, in 2012. His research interests include stream data mining, adaptive learning under concept drift and evolving data. He has published several papers in international journals and conferences. \vspace{-10 mm}
\end{IEEEbiography}

\begin{IEEEbiography}[{\includegraphics[width=1in,height=1.25in,clip,keepaspectratio]{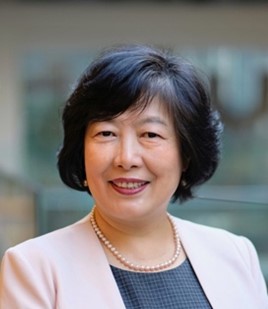}}]{Jie Lu} (F'18) is an Australian Laureate Fellow, IFSA Fellow, ACS Fellow, Distinguished Professor, and the Director of Australian Artificial Intelligence Institute (AAII) at the University of Technology Sydney, Australia. She received a PhD degree from Curtin University in 2000. Her main research expertise is in transfer learning, concept drift, fuzzy systems, decision support systems and recommender systems. She has published over 500 papers in IEEE Transactions and other leading journals and conferences. She is the recipient of two IEEE Transactions on Fuzzy Systems Outstanding Paper Awards (2019 and 2022), NeurIPS2022 Outstanding Paper Award, Australia's Most Innovative Engineer Award (2019), Australasian Artificial Intelligence Distinguished Research Contribution Award (2022), Australian NSW Premier's Prize on Excellence in Engineering or Information \& Communication Technology (2023), and the Officer of the Order of Australia (AO) 2023. \vspace{-10 mm}
\end{IEEEbiography}

\begin{IEEEbiography}[{\includegraphics[width=1in,height=1.25in, clip,keepaspectratio]{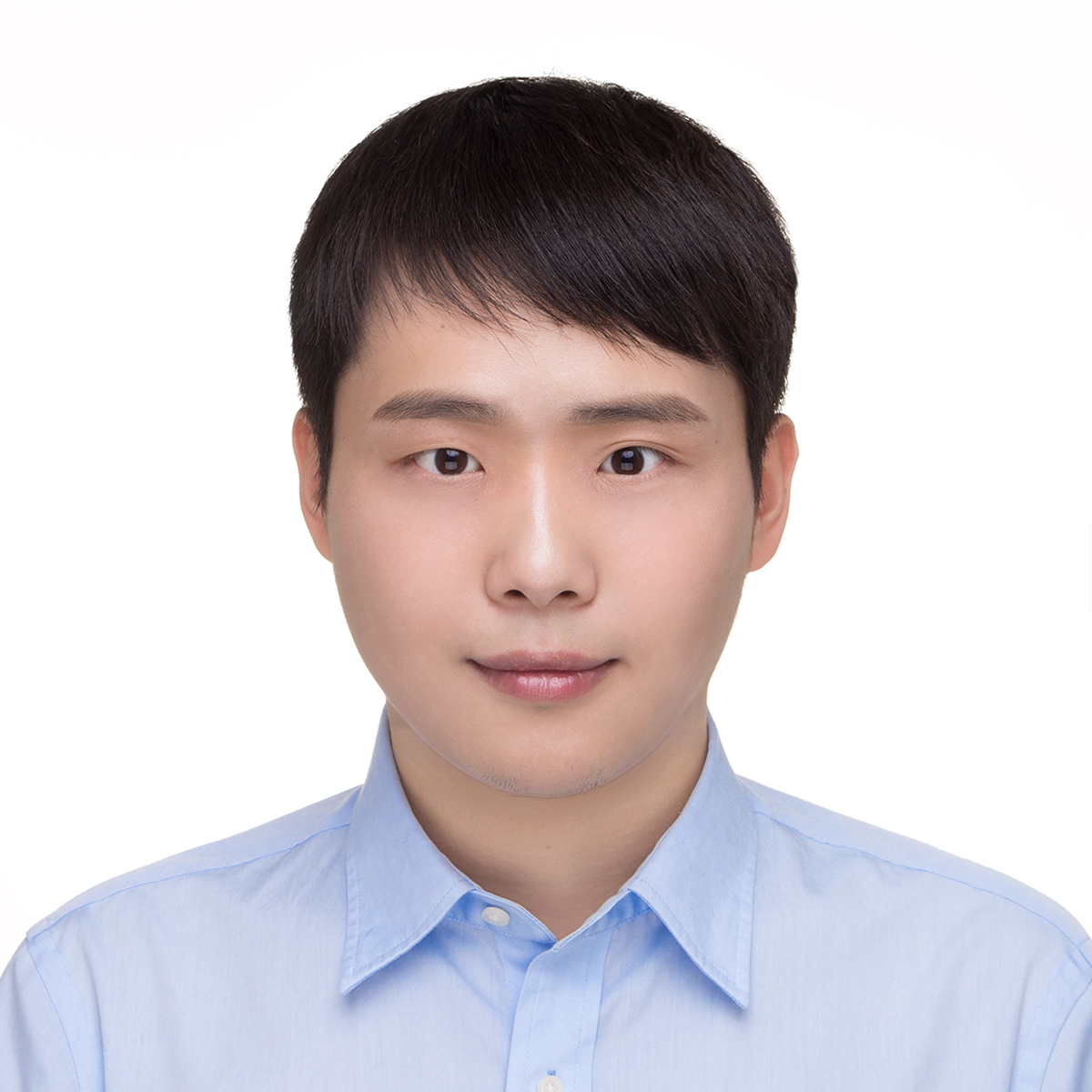}}]{Zhen Fang}
is a lecturer at Australian Artificial Intelligence Institute (AAII), University of Technology Sydney, Australia. He is a member of the Decision Systems and e-Service Intelligence (DeSI) Research Laboratory, Australian Artificial Intelligence Institute (AAII), University of Technology Sydney. His research interests include transfer learning and out-of-distribution learning. He has published several high quality papers on transfer learning and out-of-distribution learning in top conferences and journals, such as NeurIPS, ICML, ICLR, JMLR and TPAMI. He has received the NeurIPS2022 Outstanding Award and the 2023 Australasian AI Emerging Researcher Award.\vspace{-10 mm}
\end{IEEEbiography}

\begin{IEEEbiography}
[{\includegraphics[width=1in,height=1.25in, clip,keepaspectratio]{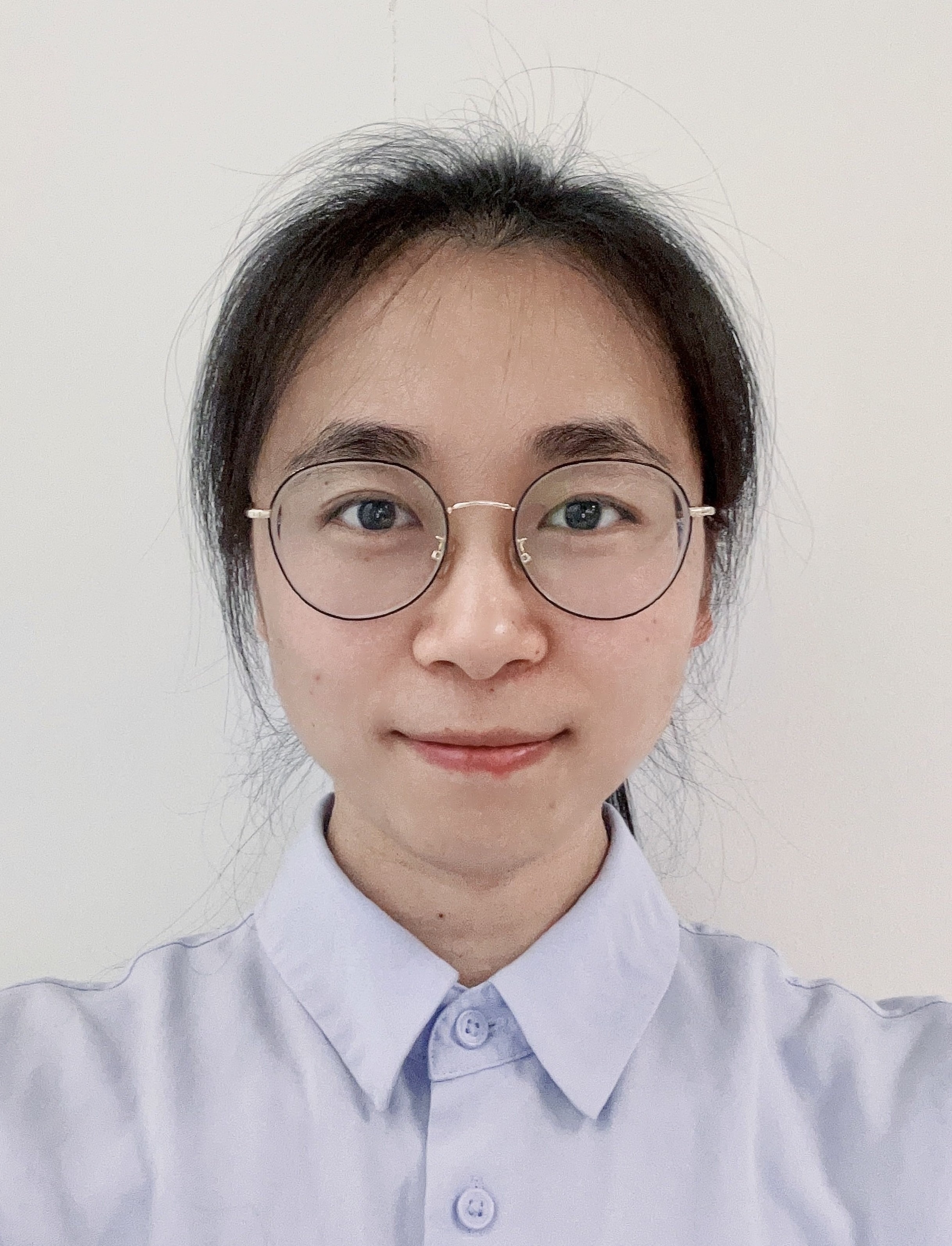}}]{Kun Wang} is a Postdoctoral Research Fellow at Australian Artificial Intelligence Institute (AAII), University of Technology Sydney. She received the Ph.D. degree in Computer Science from the University of Technology Sydney in 2024, and the Ph.D. degree in Management Science and Engineering from Shanghai University in 2023. Her research interests include concept drift adaptation, data stream mining and information management. She has published 10 papers in related areas.
\end{IEEEbiography}

\begin{IEEEbiography}[{\includegraphics[width=1in,height=1.25in, clip,keepaspectratio]{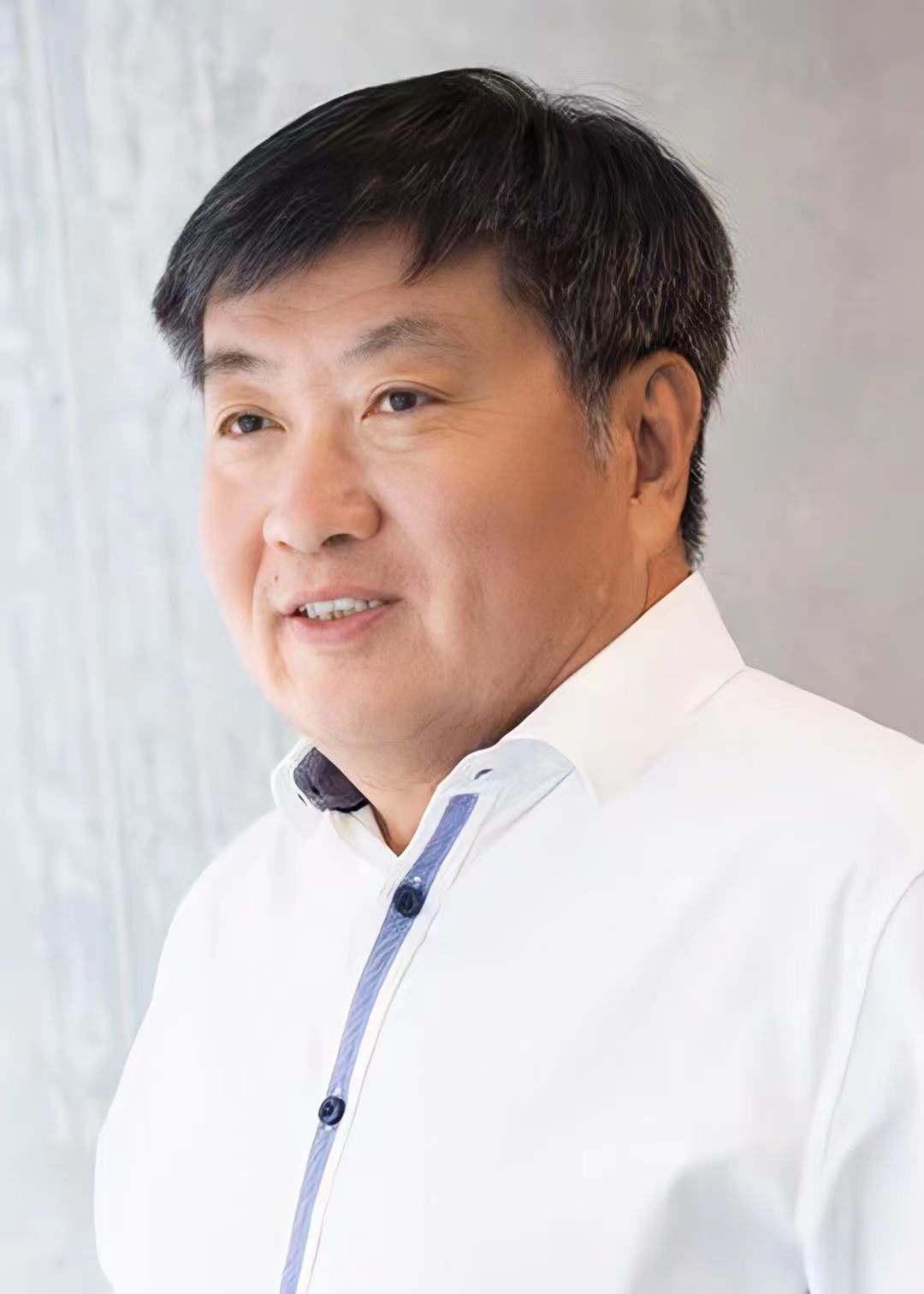}}]{Guangquan Zhang}
is an Australian Research Council (ARC) QEII Fellow, Associate Professor, and the Director of the Decision Systems and e-Service Intelligence (DeSI) Research Laboratory at the Australian Artificial Intelligence Institute, University of Technology Sydney, Australia. He received a Ph.D. in applied mathematics from Curtin University, Australia, in 2001. From 1993 to 1997, he was a full Professor in the Department of Mathematics, at Hebei University, China. His main research interests lie in fuzzy multi-objective, bilevel and group decision-making, fuzzy measures, transfer learning, and concept drift adaptation. He has published six authored monographs and over 500 papers including some 300 articles in leading international journals. He has supervised 40 Ph.D. students to completion and mentored 15 Postdoc fellows. Prof. Zhang has been awarded ten very competitive ARC Discovery grants and many other research projects. His research has been widely applied in industry.
 \vspace{-10 mm}
\end{IEEEbiography}

\end{document}